%% file: main.tex
\documentclass[lettersize,onecolumn]{IEEEtran}

\usepackage{amsthm}
\usepackage{amsmath,amsfonts}
\usepackage{algorithmic}
\usepackage{algorithm}
\usepackage{array}

\usepackage{caption}
\usepackage{comment}
\usepackage{subcaption}
\usepackage{graphicx}

\usepackage{textcomp}
\usepackage{stfloats}
\usepackage{url}
\usepackage{verbatim}
\usepackage{graphicx}
\input{Bookeeping/0_Information}
\input{Bookeeping/1_macros_and_packages}
\input{Bookeeping/2_mathcommands}

\newcommand{\code}{%
    \url{https://github.com/AnastasisKratsios/Risk_Bound_Ablation.git}
}

    \definecolor{blueviolet}{rgb}{0.54, 0.17, 0.89}
     \definecolor{amethyst}{rgb}{0.6, 0.4, 0.8}

\begin{document}

\title{Tighter Learning Guarantees on Digital Computers via Concentration of Measure on Finite Spaces}

\author{Anastasis Kratsios$^1$\thanks{$^1$ Equal Contribution. Department of Mathematics, McMaster University and The Vector Institute}, A. Martina Neuman$^2$ \thanks{$^2$ Equal Contribution. Faculty of Mathematics, University of Vienna}, Gudmund Pammer$^3$ \thanks{$^3$ Department of Mathematics, ETH Z\"{u}rich}
}

\markboth{Submitted to: IEEE Transactions on Information Theory - March 2024}%
{Shell \MakeLowercase{\textit{et al.}}: A Sample Article Using IEEEtran.cls for IEEE Journals}

\maketitle

\begin{abstract}   
Machine learning models with inputs in a Euclidean space $\mathbb{R}^d$, when implemented on digital computers, generalize, and their generalization gap converges to $0$ at a rate of $c/N^{1/2}$ concerning the sample size $N$. However, the constant $c>0$ obtained through classical methods can be large in terms of the ambient dimension $d$ and machine precision, posing a challenge when $N$ is small to realistically large.  In this paper, we derive a family of generalization bounds $\{c_m/N^{1/(2\vee m)}\}_{m=1}^{\infty}$ tailored for learning models on digital computers, which adapt to both the sample size $N$ and the so-called geometric representation dimension $m$ of the discrete learning problem.
Adjusting the parameter $m$ according to $N$ results in significantly tighter generalization bounds for practical sample sizes $N$, while setting $m$ small maintains the optimal dimension-free worst-case rate of $\mathcal{O}(1/N^{1/2})$.  Notably, $c_{m}\in \mathcal{O}(m^{1/2})$ for learning models on discretized Euclidean domains.
Furthermore, our adaptive generalization bounds are formulated based on our new non-asymptotic result for concentration of measure in finite metric spaces, established via leveraging metric embedding arguments.
\end{abstract}


\section{Introduction}
\label{s:Introduction}
Many mathematical explanations for the success of machine learning models in solving high-dimensional problems operate under the simplified modeling assumption that the models map between ``continuous'' spaces.  
For instance, many universal approximation theorems for MLPs, e.g.~\cite{mhaskar2016deep,yarotsky2017error,petersen2018optimal,shen2022optimal,acciaio2023designing}, and generalization bounds for various ML models, e.g.~\cite{neyshabur2017pac,neyshabur2018pac,JMLRBartlett_HarveyLiawMehrabian_RiskBounds,bartlett2017spectrally,IEEE_Generalization_2021,aminian2021information,MeiMontanariCPA_2022,IEEE_Generalization_2022,zhou2023exactly,gonon2023approximation,cheng2023theoretical,NgYip_AA_2023__GNNsEigenRisk,hou2022instance,MR4583284,IEEE_GeneralizationMLP_2023,barzilai2023generalization,cheng2024}, assume that the inputs of the respective learning model belong to a positive-measured compact subset of a Euclidean space $\mathbb{R}^d$.
These stylized mathematical assumptions fail to address the constraints imposed on machine learning models by standard digital computers. These structural constraints are implicitly enforced by a variety of factors ranging from software limitations, e.g.~finite machine precision and the limits of floating point arithmetic \cite{GoldbergFloatingPointArithmitic_ACM_1991,muller2018handbook}, to hardware limitations, e.g.~finite Random Access Memory (RAM) in contemporary Graphics Processing Units (GPUs) or Central Processing Units (CPUs).

Simultaneously, it is well-known
that accounting for machine precision offers a means to circumvent {\it the curse of dimensionality} inherent in high-dimensional learning with 
$N$ samples, reducing from the learning rate of $\mathcal{O}(N^{1/(2\vee d)})$ to the parametric rate of $\mathcal{O}(N^{1/2})$; 
see e.g.~\citep[Remark 4.1 and Corollary 4.6]{shalev2014understanding}. However, generalization bounds accounting for machine precision derived using these learning theoretic tools 
can be quite loose when 
the number of training samples $N$ is not massive.  
This occurs because 
the majorant constants derived for the sample complexity using the classical methods, e.g.~\citep[Corollary 4.6]{shalev2014understanding}, can be very large and even depend 
on the dimension $d$. Consequently, the anticipated shortcut through leveraging digital computing considerations often becomes imperceptible for {\it moderate and practical} sample sizes. 

One of our key findings (Theorem~\ref{thm:MAIN_no_partition}) demonstrates that digital computing can yield non-trivial improvements to the theoretical generalization bounds in the regime where $N$ is small-to-large but not massive, as illustrated in Figure~\ref{fig:Phase_Diagram}.
Notably, our depicted majorant constants are determined solely by an adjustable geometric {\it representation dimension} $m$ of the discretized learning problem, 
{\it rather than by} $d$. 
We accomplish this by designing a family of generalization bounds for different values of $m$, all concurrently valid and applicable for various sample sizes. An optimal bound can thus be determined by effectively balancing the majorant constant value and the corresponding guaranteed rate. 
Moreover, we compute these generalization bounds using our other key finding (Theorem~\ref{thrm:ConcentrationFiniteMetricSpaces}), which introduces a novel concentration result for empirical probability measures on finite metric spaces, assessed in the $1$-Wasserstein distance. 

     \begin{figure}[H]
        \centering
        \includegraphics[width=0.5\linewidth]{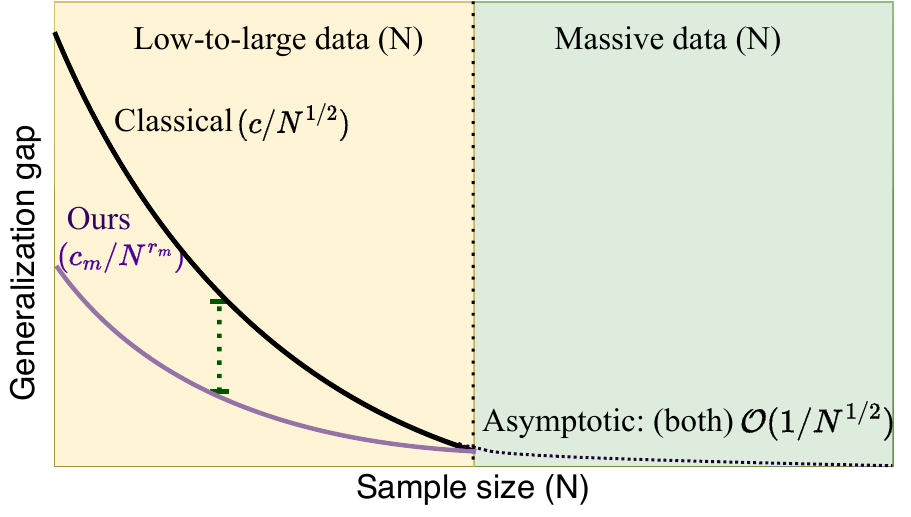}
        \caption{When the sample size $(N)$ is small-to-realistically-large, our non-asymptotic risk bounds are tighter than the classical bounds, e.g.~\citep[Corollary 4.6]{shalev2014understanding}).  
        For massive sample sizes $N$, both bounds yield the {theoretical}, parametric rate of $\mathcal{O}(1/N^{1/2})$.
        See Section~\ref{s:Discussion_PAC} for theoretical and numerical demonstrations of this phenomenon. }
        \label{fig:Phase_Diagram}
     \end{figure}

{Specifically, we investigate the statistical ramifications that digital computers may impose on machine learning models in the presence of noise in the samples. Particularly, our formulation of noise is sufficiently flexible to be interpreted as bridging classical \textit{nonparametric regression} and more general \textit{agnostic}, \textit{probably-approximately-correct (PAC)}-style frameworks.}
{We now introduce the central terminology for our problem before turning to the main text.}

\paragraph{The learning problem}
A machine learning model implemented on a digital computer processes inputs and outputs points on 
\begin{equation} \label{eq:Rdpm}
\mathbb{R}_{p,M}^d \eqdef
    \Big\{ \Big(
        \frac{a_1}{2^{j_1}},
        \dots,
        \frac{a_d}{2^{j_d}}
    \Big)
        \in \mathbb{R}^d
    : a_1,\dots, a_d \in \{-M,\dots,M\} \text{ and } j_1,\dots, j_d\in\{0,\dots,p\}\Big\},
\end{equation}
for some $p,M\in\mathbb{N}$ implicitly enforced by the digital computer\footnote{In standard double-precision floating point arithmetic, $M$ is called the \textit{mantissa}, and $p$ is called the \textit{power}.}. 
Such a grid $\mathbb{R}_{p,M}^d$ is a uniform \textit{discretization} of the Euclidean cube $[-M,M]^d$ given by a dyadic grid whose adjacent points are exactly $\frac{1}{2^p}$ apart.

We consider an unknown \textit{target function} $f^{\star}:\xxx\to\yyy$ defined between discretized Euclidean subsets $\xxx$ and $\yyy$,  
a \textit{data generating probability measure} $\mathbb{P}$ on $\xxx\times \yyy$ concentrated near the graph of $f^{\star}$, and independent {\it noisy} samples 
\begin{equation} \label{eq:noisysamples}
    (X_1,Y_1), \dots ,(X_N,Y_N)\sim \mathbb{P},
\end{equation}
with the noise only obscuring the output components $Y_n${; see Section~\ref{s:Main__ss:Gen} for a rigorous definition}.
Our objective is to derive upper bounds for
\begin{enumerate}
    \item[(i)] the worst-case \textit{generalization gap} $\sup_{\hat{f}\in\mathcal{F}} |\hat{\mathcal{R}}(\hat{f})-\mathcal{R}(\hat{f})|$,
    \item[(ii)] the worst-case \textit{{reconstruction} gap} $\sup_{\hat{f}\in\mathcal{F}} |\hat{\mathcal{R}}(\hat{f})-\mathcal{R}^{\star}(\hat{f})|$, 
\end{enumerate}
uniformly over all $\hat{f}$ in a given \textit{hypothesis class} $\mathcal{F}\subset \yyy^{\xxx}$, respectively referred to as {\it generalization bounds} and {\it {reconstruction} bounds}.
The generalization gap $|\hat{\mathcal{R}}(\hat{f}) - \mathcal{R}(\hat{f})|$ induced by a hypothesis $\hat{f}$ signifies the disparity between {its \textit{empirical risk} $\hat{\mathcal{R}}(\hat{f})$ and its \textit{{population} risk} $\mathcal{R}(\hat{f})$}:
\begin{equation} \label{eqdef:trurisk}
    \hat{\mathcal{R}}(\hat{f})
    \eqdef 
    \frac{1}{N}\sum_{n=1}^N \mathcal{L}(\hat{f}(X_n),Y_n),
    \quad\text{ and }\quad
    \mathcal{R}(\hat{f})
    \eqdef 
    \mathbb{E}_{(X,Y)\sim\mathbb{P}} [\mathcal{L}(\hat{f}(X),Y)],
\end{equation}
where $\mathcal{L}:\yyy\times \yyy\to [0,\infty)$ is a \textit{loss function}. The {reconstruction} gap $|\hat{\mathcal{R}}(\hat{f})-\mathcal{R}^{\star}(\hat{f})|$ characterizes how well the model $\hat{f}$ reconstructs the target function $f^{\star}$ from the noisy training data, where the {population} risk is replaced by its noiseless counterpart called the \textit{excess risk}
\begin{equation} 
\label{eqdef:risks}
    \mathcal{R}^{\star}(\hat{f})
    \eqdef 
    \mathbb{E}_{(X,Y)\sim\mathbb{P}}[\mathcal{L}(\hat{f}(X),f^{\star}(X))].
\end{equation}
{\textit{In the absence of noise, the excess risk equates to the population risk, a reconstruction gap to a generalization gap, and a {reconstruction} bound to a generalization bound.}} 

Departing from the classical approach, our analysis pioneers a novel direction by integrating the geometry of $\xxx\times \yyy$ in the formulation 
of 
generalization and {reconstruction} bounds. We term these derived bounds as ``{\it adaptive}" due to their applicability across different geometric interpretations of $\xxx\times\yyy$ and different sample sizes. At the core of this approach is a crucial technicality: 
a similarly adaptive measure concentration inequality for finite metric spaces.

\paragraph{Concentration of measure on finite metric spaces}
\label{s:Intro__RateGeometry}

The problem of bounding the worst-case generalization gap and {reconstruction} gap can be translated into a \textit{measure concentration} problem on $\xxx\times \yyy$. 
Here, the measure concentration problem refers to the expected average distance between $\mathbb{P}$ and its random {\it empirical measure}
\begin{equation} \label{eqdef:empmeas}
    \mathbb{P}^N \eqdef \frac1{N} \sum_{n=1}^N \delta_{(X_n,Y_n)},
\end{equation}
where $\delta_{(X_n,Y_n)}$ denotes the Dirac measure at $(X_n,Y_n)$.  
By the Glivenko-Cantelli Theorem, see \citep[Chapter 2.4]{varderVaartWellner_EmpiricalProcessesBook_1996}, $\mathbb{P}^N$ converges to $\mathbb{P}$ weakly as $N$ becomes large. Further, the rate at which $\mathbb{P}^N$ converges to $\mathbb{P}$ is quantified by the $1$-{\it Wasserstein distance} $\mathcal{W}(\mathbb{P},\mathbb{P}^N)$ from optimal transport theory \cite{kantorovich1958space}. We opt for 
the Wasserstein distance since it metrizes the topology of convergence in distribution on $\xxx\times \yyy$, 
in contrast to many other statistical ``distances" such as the $f$-divergence. 
Precisely, the geometry of $\xxx\times \yyy$ is embedded within $\mathcal{W}(\mathbb{P},\mathbb{P}^N)$, which captures the distance covered in $\xxx\times \yyy$ when efficiently transporting mass from $\mathbb{P}$ to $\mathbb{P}^N$.

On the one hand, if $\mathbb{P}$ and $\mathbb{P}^N$ are probability measures on a $m$-dimensional Riemannian manifold, or a normed space \cite{Kloeckner_2020_CounterCurse}, then the \textit{Wasserstein concentration rate} $\mathcal{W}(\mathbb{P},\mathbb{P}^N)$ is understood. When $m=1$, $\mathcal{W}(\mathbb{P},\mathbb{P}^N)$ converges to $0$ at a (Monte-Carlo) rate of $\mathcal{O}(1/N^{1/2})$, and when $m\ge 3$, a rate of $\mathcal{O}(1/N^{1/m})$ \cite{fournier2015rate}. When $m=2$, a critical case occurs where the optimal rate is shown to be $\mathcal{O}(\log(N)/N^{1/2})$; see \cite{AjtaiKomlosTusnady_1984__Combinatoria}.
Furthermore, these rates are sharp \cite{GrafLuschgy_2000_FoundationsQuantizationofProbMeasure,Kloeckner_2012_QuantizationAlhforsRegularity,LiuPages_Quantization_JMLR,WeedBach_Concentration_2019__BernoulliOptimal}.

On the other hand, while the metric space $\xxx\times \yyy$ is finite, meaning that its \textit{topological} dimension is $0$, its dimension as a metric space is much more nuanced. This manifests in the fact that, unlike the previously mentioned positive-dimensional spaces, $\xxx\times \yyy$ can be injectively mapped into any Euclidean space $\mathbb{R}^m$ of a representation dimension $m$. 
Indeed, observe that any enumeration $\xxx\times \yyy=\{(x_i,y_i)\}_{i=1}^k$ induces a map $\xxx\times \yyy\ni (x_i,y_i)\mapsto i \in \mathbb{R}$ and since $\mathbb{R}\subset \mathbb{R}^m$, we have an injection of $\xxx\times \yyy$ into any $\mathbb{R}^m$.
One can therefore ask, which dimension can be used to obtain the tightest, worst-case, upper bound on $\mathcal{W}(\mathbb{P},\mathbb{P}^N)$ for a given $N$?  

As evident from the preceding discussion, the choice of $m=1$ yields a bound $\mathcal{W}(\mathbb{P},\mathbb{P}^N)\le c_1/N^{1/2}$, with high probability. The constant $c_1>0$, independent of $N$, encodes the {\it distortion} of distances between points in $\xxx\times \yyy$ that arises from its representation in $\mathbb{R}$. Such distortion can be considerable, even if $\xxx\times \yyy$ contains only three points, as illustrated by Figure~\ref{fig:DistIllustration}.
Therefore, it is conceivable that $c_1$ can become exceptionally large when $\xxx\times \yyy$ is a subset of $\mathbb{R}^d_{p,M}$ in ~\eqref{eq:Rdpm}. 

\begin{figure}[H]
\caption{The distortion incurred when compressing {a} $3$-point subset of $\mathbb{R}^2$, illustrated by Figure~\ref{fig:DistIllustration__NoDist}, into a $3$-point subset of the real line $\mathbb{R}$, intuitively illustrated by Figure~\ref{fig:DistIllustration__Dist}, 
results from the necessary shrinking or stretching of distances between the points.
}
\label{fig:DistIllustration}
     \centering
     \begin{subfigure}[b]{0.45\textwidth}
         \centering
         \includegraphics[width=.45\linewidth]{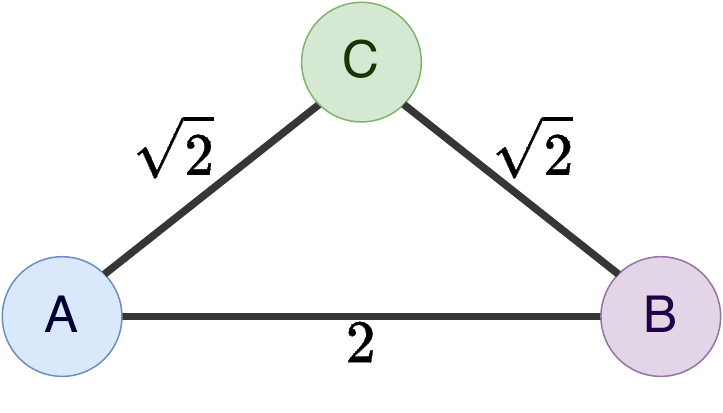}
         \caption{Consider three points $A = (0,0)$, $B = (2,0)$, and $C= (1,1)$ in $\mathbb{R}^2$. The distance between $A$ and $B$ is $2$, and 
         the distance 
         between $A$ (resp. $B$) and $C$ is $\sqrt{2}$.
         }
        \label{fig:DistIllustration__NoDist}
     \end{subfigure}
     \hfill
     \begin{subfigure}[b]{0.45\textwidth}
         \centering
         \includegraphics[width=.45\linewidth]{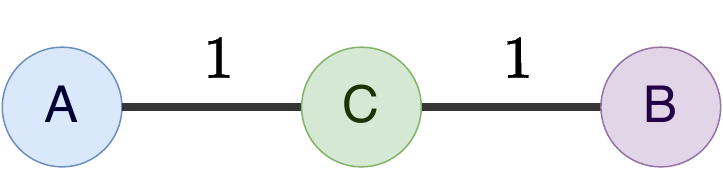}
         \caption{
         Suppose we embed $\{A,B,C\}$ into $\mathbb{R}$, sending $A\mapsto 0$, $B\mapsto 2$, $C\mapsto 1$, then the distance between $A$ and $B$ is kept, but the distance between $A$ (resp. $B$) and $C$ is distorted by a factor of $\sqrt{2}$.}
         \label{fig:DistIllustration__Dist}
     \end{subfigure}
\end{figure}

At the other extreme, 
high-dimensional representations of $\xxx\times \yyy$, i.e.\ when $m\gg 1$, deliver high-probability bounds of the form, $\mathcal{W}(\mathbb{P},\mathbb{P}^N)\le c_m/N^{1/m}$. In this case, $c_m>0$ is typically much smaller than $c_1$. However, the convergence rate of $1/N^{1/m}$ is significantly slower than that obtained using one-dimensional representation. The representation dimension parameter $m$ thus enables 
one 
to modulate the trade-off between the adaptive constant $c_m$ and  the convergence rate $1/N^{1/m}$.  
These observations lead to the following informal version of our result regarding the concentration of measure when $\xxx$, $\yyy$ are discretized Euclidean subsets. To streamline the presentation, we incorporate the constant $c_m$ into the notation $\mathcal{O}$. The value of this constant can be calculated explicitly using the details in Theorem~\ref{thrm:ConcentrationFiniteMetricSpaces}.  


\begin{InformalTheorem}[Concentration of measure on finite grids]
\label{informaltheorem_Concentration} \hfill\\
Let $d$, $p$, $M$, $k$, $N\in \mathbb{N}$. Let $\mathbb{P}$ be a probability measure supported on a $k$-point subset of $\mathbb{R}^d_{p,M}\times \mathbb{R}^1_{p,M}$, and let $\mathbb{P}^N$ be its empirical version defined by $N$ i.i.d.\ samples.~
Then with {probability $1-2^{-\Theta(N)}$}, $\mathcal{W}(\mathbb{P},\mathbb{P}^N)$ is bounded above by\footnote{{We omit $m=2$ here because the corresponding bound incurs a logarithmic factor in $N$ (see Table~\ref{tab:concentration_main_RATES}); in this informal statement, we aim to present a uniform expression.}}
\begin{equation*}
    \min_{\underset{m\neq 2}{m\in \mathbb{N}}} \underbrace{\frac1{N^{1/(m\vee 2)}}}_{\text{adaptive rate}}
        \begin{cases}
            \mathcal{O}(k) & \mbox{ if } m=1\\
            \mathcal{O}\Big(\frac{\log(k)^{1/2} k^{2/m}}{m^{1/2}}\Big) & \mbox{ if } 3\le m \leq \lceil 8 \log(k)\rceil\\
            \mathcal{O}(\log(k)) & 
            \mbox{ if } \lceil 8 \log(k) \rceil < m < d\\
            1 & \mbox{ if } {d}\le m.
        \end{cases}
\end{equation*}
\end{InformalTheorem}

\subsection{Summary of contributions}
\label{sec:overview}

We present two primary contributions, namely Theorems~\ref{thrm:ConcentrationFiniteMetricSpaces} and \ref{thm:MAIN_no_partition}. Theorem~\ref{thm:MAIN_no_partition} provides adaptive agnostic PAC-learning guarantees applicable to hypothesis classes defined between finite metric spaces, while Theorem~\ref{thrm:ConcentrationFiniteMetricSpaces}, a key tool in deriving Theorem~\ref{thm:MAIN_no_partition}, delivers an adaptive measure concentration result for probability measures in a finite metric space, quantified in $1$-Wasserstein distance. 
The adaptability is manifested in the fact that both generalization and {reconstruction} bounds, as well as the concentration rates, tailor themselves to the metric geometry and the sample size. 

Upon applying our results to multilayer perceptions with ReLU activation functions (ReLU MLPs), we unveil the consistent breaking of the curse of dimensionality in regression analysis when these machine learning models are executed on digital computers. We conclude that the constraints imposed by digital computing may account for 
the real-world success of machine learning, potentially explaining the divergence 
with pessimistic theoretical outcomes in statistical learning theory. While this relationship between the ``discretization trick'' \cite[Remark~4.1]{shalev2014understanding} and curse of dimensionality has been recognized, our results provide explicit evidence that they can yield tighter generalization bounds for classifiers on finite spaces than what can be derived from VC theory for realistic sample sizes $N$.

From a technical standpoint, our theoretical framework introduces new techniques for deriving measure concentration rates and learning theoretic guarantees that are grounded in metric embedding theory and optimal transport. The theory illuminates the geometric trade-off mentioned in \S\ref{s:Intro__RateGeometry} between the Wasserstein concentration rate and the representation dimension.
Moreover, we obtain generic worst-case distortion bounds for Euclidean metric embeddings, complete with explicit majorant constants, drawing upon classical findings from \cite{BourgainEmbedding_Original_1985} and \cite{Matouvsek_1996_Embeddings}.

Finally, as our main results are formulated for general metric spaces, they enable the derivation of generalization bounds for various contemporary machine learning models. 
Examples include graph neural networks which are defined on spaces of graphs~\cite{zhang2018link}, generative models which produce probability measures as outputs~\cite{arjovsky2017Wasserstein,cao2019multi,xu2020cot}, to differential equation solvers defined between infinite-dimensional linear spaces \cite{chen2018neural,herrera2020neural,morrill2021neural,kovachki2021universal,karniadakis2021physics}.

\subsection{Outline of paper} 
Our paper is structured as follows\footnote{The code used to generate figures in our paper is given at \code.}.
Section~\ref{sec:conventions} provides the necessary background, definitions, and notations, along with a preliminary Euclidean embedding result.
Section~\ref{s:Main} presents our two main theorems.
Section~\ref{s:Discussion} explores the implications of our theorems for regression analysis and binary classification over finite metric domains.
Section~\ref{s:Applications} focuses on an application of our results in deep learning. 
All detailed proofs are provided in Section~\ref{sec:Proofs}. Finally, short, extra discussions are given in the Appendix.

\subsection{Comparison with related work}
\label{s:Related_Work}

\paragraph{Comparison with VC and information theoretic results}

Our results diverge from the classical theory regarding statistical learning guarantees and concentration of measure primarily because they possess the unique capability to exploit the geometry of $\xxx\times\yyy$. Specifically, conventional statistical learning outcomes for classifiers on finite spaces \cite{Talagrand_SharperBoundEmpGausProcess_1994, Long_1997ComplexityPAC_ML1998, AryehIosif_2019_ExactPACMLFiniteCase_2019} hinge on set-theoretic tools, such as VC dimension, which do not capture the metric structure. The same is true for standard information-theoretic concentration of measure results on finite sets \cite{MardiajiaoTanczosNowakWeissman__2020_IAI_ConcentrationEmpiricalKL} which rely on tools, such as relative entropy, that do not encode the metric information of the underlying space.  
Consequently, a scenario unfolds in which information-theoretic and VC-type bounds are 
tight only for immensely large sample sizes $N$, and yet prove to be conservative for pragmatically large values of $N$, as compared to our bounds. 

An example highlighting the fact that classical information-theoretic divergences cannot capture metric structure is as follows.
Let $\mathbb{P}=\delta_0$ and $\mathbb{Q}_j$ be the uniform distribution on $\{0,j\}$, for $j\in\mathbb{N}$. 
Then $\mathcal{W}(\mathbb{P},\mathbb{Q}_j) = \frac{j}{2}$ but $\operatorname{KL}(\mathbb{P}|\mathbb{Q}_j)=\log(2)$.  
Thus, the KL-divergence cannot detect the distance between the points $\{0,j\}$ while the Wasserstein distance depends linearly on the distance between $0$ and $j$.  
The same holds for \textit{most} $f$-divergences due to their representation in terms of the $E_{\gamma}$-divergence shown in~\cite[Proposition 3]{sason2016f}. 
Thus, such $f$-divergences, between $\mathbb{P}$ and $\mathbb{Q}_j$, do not depend on $j$.

\paragraph{Concentration of measure on finite metric spaces} 
We note that \cite{sommerfeld2018inference} also studied the convergence of $\mathbb{P}^N$ to $\mathbb{P}$ over finite metric spaces. However, their analysis primarily addresses situations where the sample size $N$ is exceptionally large and does not optimize their bounds for realistic $N$. Moreover, outcomes related to the concentration of measure in finite metric spaces can be inferred from findings applicable to generalized doubling metric spaces \cite{boissard2014mean,WeedBach_Concentration_2019__BernoulliOptimal}. While these results are broad in scope, their provided rates explicitly rely on the doubling constant of the underlying metric space.  
From a practical standpoint, computing these constants as well as the doubling dimensions of many finite spaces is often challenging. This can be seen in~\cite{durand2021least,DurandCartagenaEstibalitzSoraTradacete_2023_DiscreteMath}, as well as in~\citep[Lemma 21]{kratsios2022small}, where even relatively simple finite unweighted graphs demand considerable effort to attain accurate estimates of these quantities.  

\paragraph{The impact of digital computing}
The impact of digital computing on machine learning and scientific computing problems has recently come into focus in the signal processing, approximation theory, and inverse problems literature. 
Notable examples include the efficiency of sampling-quantization algorithms in signal processing for bandlimited signals \cite{DaubechiesDeVore_ContinuousToDigitalSamplingQuantization_2003__AnnMath,DGunturk_ImprovedRatesContinuousToDigitalSamplingQuantization_2003__AMS}, optimal compression rates in rate-distortion theory~\cite{GraphKlotzVoigtlaender_2023_RDTDL__FOCM}, recent literature demonstrating the (in)feasibility of solving various inverse problems on digital hardware \cite{boche2022inverse, boche2022limitations,boche2022non}, and limited research on expressivity of neural networks implemented on digital computers \cite{park2024expressive}. These works primarily investigate the \textit{analytic} effects of machine precision, e.g.\ on approximation, whereas we focus on the \textit{statistical} ramifications of digital computing.  

\section{Preliminary} \label{sec:conventions}

We introduce the necessary tools and terminology to formulate our main results {in} Section~\ref{sec:background}; this includes symbols, notations, and conventions that will be consistently employed throughout this paper.  Section~\ref{sec:Euclideanembedding} further summarizes our technical sharpening of the available worst-case results on the best embedding of any finite metric space in Euclidean spaces of any prescribed dimension.
It is assumed that the reader is familiar with fundamental concepts in analysis and probability theory.

\subsection{Background and notation} \label{sec:background}
 
\paragraph{On metric spaces} 
Given a metric space $(\xxx,d_{\xxx})$, consisting of a non-empty set of points $\xxx$ and a metric $d_{\xxx}$, its  \textit{diameter} is defined as
\begin{equation*} 
   \mathsf{d}(\xxx) \eqdef \sup_{x,y\in\xxx} d_{\xxx}(x,y).
\end{equation*}
The metric space $(\xxx,d_{\xxx})$ is called a finite metric space if $\xxx$ has finite cardinality, i.e. $\mathrm{card}(\xxx)<\infty$. An example of a metric space is a subset of $\mathbb{R}^m$, inheriting the ambient Euclidean distance induced by the Euclidean $\|\cdot\|_2$ norm.

Let $(\xxx, d_{\xxx})$ and $(\yyy, d_{\yyy})$ be two metric spaces (typically representing the input/source space and the output/target space, respectively).  
Their Cartesian product $\xxx\times \yyy$ is metrized using the 
usual product metric
\begin{equation} \label{prodmetric}
    d_{\xxx\times\yyy}((x_1,y_1),(x_2,y_2)) 
    \eqdef d_{\xxx}(x_1,x_2) + d_{\yyy}(y_1,y_2).
\end{equation}

\paragraph{On Lipschitz mappings of finite metric spaces}
Let $\varphi:{\xxx} \to {\yyy}$ denote a map from a \textit{finite} metric space $({\xxx},d_{\xxx})$ to a metric space $({\yyy},d_{\yyy})$. 
We define the \textit{lower Lipschitz} and \textit{upper Lipschitz} constants of $\varphi$ to be, respectively,
\begin{equation*}
    {\rm L}_{\ell}(\varphi) \eqdef \min_{ \substack{ x,\tilde{x}\in\xxx \\ x\neq \tilde x } } \frac{d_{\yyy}\big(\varphi(x),\varphi(\tilde{x})\big)}{d_{\xxx}(x,\tilde{x})} \quad\text{ and }\quad
    {\rm L}_{\mathit{u}}(\varphi) \eqdef \max_{ \substack{ x,\tilde{x}\in\xxx \\ x\neq \tilde x } } \frac{d_{\yyy}\big(\varphi(x),\varphi(\tilde{x})\big)}{d_{\xxx}(x,\tilde{x})}.
\end{equation*}
If ${\rm L}_{\mathit{u}}(\varphi)<\infty$, then $\varphi$ is called a \textit{Lipschitz} map. 
If in addition,
\begin{equation*} 
    0<{\rm L}_{\ell}(\varphi)\leq {\rm L}_{\mathit{u}}(\varphi)<\infty,
\end{equation*}
then $\varphi$ is called a \textit{bi-Lipschitz embedding} ({\it map}) of $\xxx$ into $\yyy$, in which case
\begin{equation*}
    {\rm L}_{\ell}(\varphi) d_{\xxx}(x,\tilde{x})
    \leq 
    d_{\yyy}(\varphi(x),\varphi(\tilde{x}))
    \leq {\rm L}_{\mathit{u}}(\varphi) d_{\xxx}(x,\tilde{x}).
\end{equation*}
The {\it distortion} of $\xxx$ under a bi-Lipschitz map $\varphi$, capturing how much $\varphi$ shrinks or expands points in $\xxx$, is given by 
\begin{equation*}
    \tau(\varphi)\eqdef \frac{{\rm L}_{\mathit{u}}(\varphi)}{{\rm L}_{\ell}(\varphi)}\geq 1.
\end{equation*}
The reader can compare this distortion definition with the \textit{Lipschitz distance} between finite metric spaces in \cite{BourgainEmbedding_Original_1985, Matouvek_OptimalEuclidean_Npoint}.

\paragraph{On probability spaces} A complete metric space $(\xxx,d_{\xxx})$ equipped with its Borel $\sigma$-algebra $\mathcal{B}_{\xxx}$ is a \textit{standard Borel space}, denoted by $(\xxx,\mathcal{B}_{\xxx})$. 
Note that when $(\xxx,d_{\xxx})$ is finite, thus complete, $\mathcal{B}_{\xxx}$ coincides with its powerset $\{0,1\}^{\xxx}$. 
Denote by $\mathcal{P}(\xxx)$ the set of all Borel probability measures on $\xxx$. 
Then, for each $\nu\in\mathcal{P}(\xxx)$,  $(\xxx,\mathcal{B}_{\xxx},\nu)$ is referred to as a \textit{Borel probability space}.

On a Borel probability space $(\xxx,\mathcal{B}_{\xxx},\nu)$, let $X_1, \dots, X_N\sim\nu$ be a finite sequence of identical independent (i.i.d.) random variables taking values in $\xxx$. 
Then, similar to \eqref{eqdef:empmeas}, the \textit{random empirical measure} $\nu^N\in\mathcal{P}(\xxx)$ associated with $\nu$ is
\begin{equation*} 
    \nu^N \eqdef \frac{1}{N} \sum_{n=1}^N \delta_{X_n}.
\end{equation*}

\paragraph{On distances between probability measures on finite metric spaces} 
Let $(\xxx,d_{\xxx})$ be a finite metric space, and let $\mathbb{Q}, \tilde{\mathbb{Q}}\in\mathcal{P}(\xxx)$.
Let $\gamma$ denote a {\it coupling} between $\mathbb{Q}$, $\tilde{\mathbb{Q}}$; namely, $\gamma\in\mathcal{P}(\xxx\times\xxx)$ satisfying $\pi^1_{\texttt{\#}}\gamma= \mathbb{Q}$ and $\pi^2_{\texttt{\#}}\gamma= \tilde{\mathbb{Q}}$. Here, $\pi^1$ and $\pi^2$ are the canonical projections of $\xxx\times \xxx$ onto its first and second components, respectively. 
The ($1$-)Wasserstein distance between $\mathbb{Q}, \tilde{\mathbb{Q}}$ on $(\xxx,d_{\xxx})$ is defined by
\begin{equation} \label{eqdef:Wassdistance}
    \mathcal{W}_{\xxx}(\mathbb{Q},\tilde{\mathbb{Q}}) \eqdef \inf_{\gamma\in\Gamma(\mathbb{Q},\tilde{\mathbb{Q}})} \sum_{x\in \xxx}\sum_{y\in \xxx} d_{\xxx}(x,y)\,\gamma(x,y),
\end{equation}
where $\Gamma(\mathbb{Q},\tilde{\mathbb{Q}})$ is the set of all couplings $\gamma$ between $\mathbb{Q}$, $\tilde{\mathbb{Q}}$. 

Alternatively, we can also quantify the distance between $\mathbb{Q}, \tilde{\mathbb{Q}}$ by their \textit{total variation distance} $\mathrm{TV}(\mathbb{Q},\tilde{\mathbb{Q}})$, expressed as their largest difference ascribed to any event; formally
\begin{equation} \label{eqdef:totalvariation}
   \mathrm{TV}(\mathbb{Q},\tilde{\mathbb{Q}}) \eqdef \max_{B \subset \xxx} |\mathbb{Q}(B)-\tilde{\mathbb{Q}}(B)|.
\end{equation}
In particular, we will use total variation to assess the disparity between a measure and its noisy 
version. This choice is advantageous due to the straightforward formulation \eqref{eqdef:totalvariation} and practical verifiability.
Furthermore, as established in \cite[Theorem~6.15]{VillaniOTBook_2009}, $\mathcal{W}_{\xxx}(\mathbb{Q},\tilde{\mathbb{Q}})$ can be bounded above by $\mathrm{TV}(\mathbb{Q},\tilde{\mathbb{Q}})$ times the diameter $\mathsf{d}(\xxx)$ of $\xxx$:
\begin{equation} 
\label{eq:Wass_TV}
    \mathcal{W}_{\xxx}(\mathbb{Q},\tilde{\mathbb{Q}})
    \le 
   \mathsf{d}(\xxx)\,\mathrm{TV}(\mathbb{Q},\tilde{\mathbb{Q}}).
\end{equation}


\subsection{Euclidean embedding of finite metric spaces} \label{sec:Euclideanembedding}
 
Central to our approach in leveraging the geometric structure of a finite metric space is the fact that it can be embedded via a bi-Lipschitz map into a Euclidean space of \textit{any} dimension, a property we now formalize with a precise statement.

\begin{proposition}[Euclidean representation of finite metric spaces] \label{prop:EuclideanRep} \hfill\\
Let $(\xxx,d_{\xxx})$ be a finite metric space with $\mathrm{card}(\xxx)=k$. Then for every $m\in\mathbb{N}$, there exists a bi-Lipschitz embedding $\varphi_m: \xxx\to\mathbb{R}^m$ whose distortion $\tau(\varphi_m)$ adheres to the following conditions.
\begin{itemize} 
    \item If $m=1,2$, then $\tau(\varphi_m)\leq 12k$.
    \item If $2< m \le \lceil 8\log (k)\rceil$, then
    \begin{equation} \label{mid1}
        \tau(\varphi_m)\leq 720 \lfloor 2\log (k)+1\rfloor k^{2/m} \tilde{\varepsilon}_{m,k} \frac{\sqrt{\log (k)}}{\sqrt{m}}.
    \end{equation}
    \item If $\lceil 8\log (k)\rceil < m\leq 2^{k}-1$, then
    \begin{equation} \label{mid2}
        \tau(\varphi_m)\leq 48\lfloor 2\log (k) +1\rfloor\tilde{\varepsilon}_{m,k}.
    \end{equation}
    \item If $m \geq 2^{k}$, then
    \begin{equation*}
        \tau(\varphi_m)\leq 48\lfloor 2\log (k)+1\rfloor.
    \end{equation*}
\end{itemize}
Here in \eqref{mid1}, \eqref{mid2}, 
\begin{equation} \label{eqdef:tildeeps}
    \tilde{\varepsilon}_{m,k} \eqdef \frac{(m^{1/2}+2\sqrt{2} (\log (k))^{1/2})^{1/2}}{(m^{1/2}-2\sqrt{2}(\log (k))^{1/2})^{1/2}}.
\end{equation}
Suppose in addition that there exists $d\in\mathbb{N}$ such that $\xxx$ is a metric subspace of $\mathbb{R}^d$. Then, a tighter upper bound for the embedding distortion can be provided in the case $2< m \le \lceil 8\log (k)\rceil$, which is,
\begin{equation*} 
    \tau(\varphi_m)\leq 15\,k^{2/m}\Big(\frac{\log (k)}{m}\Big)^{1/2},
\end{equation*}
and in the case $m\geq d$, which is, $\tau(\varphi_m) = 1$.
\end{proposition}

{The proof of Proposition~\ref{prop:EuclideanRep} is given in Section~\ref{sec:Proofs_Discrete_Atlas}.
It relies mostly on results derived from Bourgain's metric embedding theorem~\citep{BourgainEmbedding_Original_1985} and the Johnson-Lindenstrauss lemma\footnote{{An observant reader may notice that \eqref{mid2} resembles the Johnson-Lindenstrauss lemma. Indeed, we employ the lemma for this specific range of embeddings. For other dimensional ranges, however, the lemma alone is insufficient for our purposes; see \eqref{embedding_mid_range}, \eqref{dim_condition}.}}~\citep{JohnsonLindenstrauss_1984_originalpaper}, and is further informed by ideas from~\citep{Matouvek_OptimalEuclidean_Npoint, DubhashiPanconesi_2009_Concentration, Matouvsek_2002_LecturesDiscreteGeo}.}

\section{Set-up and main results}
\label{s:Main}

We present two main results.
The first, Theorem~\ref{thrm:ConcentrationFiniteMetricSpaces}, is an adaptive concentration inequality on finite metric spaces, representing a significant standalone contribution to \textit{statistical} optimal transport on \textit{finite} metric spaces. 
It is detailed in Section~\ref{sec:Wassconcmeaspf}.
The second result, Theorem~\ref{thm:MAIN_no_partition}, establishes novel adaptive generalization and {reconstruction} bounds for a statistical learning problem on finite metric spaces. 
It follows as a direct consequence of the first and is presented in Section~\ref{s:Main__ss:Gen}.
Each subsection introduces the relevant context and concludes with a precise formalization of the respective result.

\subsection{A result on adaptive concentration of measure on finite metric spaces} \label{sec:Wassconcmeaspf}

Consider a finite metric space $(\xxx,d_{\xxx})$.
Let $\mathbb{P}\in\mathcal{P}(\xxx)$ be a Borel probability measure, and let $X_1,\dots,X_N\sim\mathbb{P}$ be a sequence of i.i.d. random variables in $\xxx$. Let $\mathbb{P}^N$ be the empirical measure defined by $(X_i)_{i=1}^N$. 
We derive the rate at which the Wasserstein distance \eqref{eqdef:Wassdistance} $\mathcal{W}_{\xxx}(\mathbb{P},\mathbb{P}^N)$ between $\mathbb{P}$, $\mathbb{P}^N$, concentrates, by capitalizing on the connection between the metric structure $(\xxx,d_{\xxx})$ and its Euclidean representation in the following Theorem~\ref{thrm:ConcentrationFiniteMetricSpaces}. 
A proof can be found in Section~\ref{sec:propCOMproof}.

\begin{theorem}[Adaptive concentration of measure on finite metric spaces]
\label{thrm:ConcentrationFiniteMetricSpaces} \hfill\\
Let $(\xxx,d_{\xxx})$ be a finite metric space, with ${\rm card}(\xxx)=k$. Let $\mathbb{P}\in\mathcal{P}(\xxx)$ be a Borel probability measure. Let $X_1,\dots,X_N\sim\mathbb{P}$ be a sequence of i.i.d.~random variables in $\xxx$, and let $\mathbb{P}^N$ be the associated random empirical measure.  
Then for every $m\in\mathbb{N}$ and every bi-Lipschitz Euclidean embedding $\varphi_m: (\xxx, d_{\xxx})\to (\mathbb{R}^m,\|\cdot\|_2)$, the following hold:\\~\\
{\rm (i)} $\mathbb{E}\big[\mathcal{W}_{\xxx}(\mathbb{P},\mathbb{P}^N)\big]\le  C_1(\varphi_m,\xxx,m,N)$; \\
{\rm (ii)} for every $\varepsilon>0$, the event that
\begin{equation*}
    \big|\mathcal{W}_{\xxx}(\mathbb{P},\mathbb{P}^N) - \mathbb{E}\big[\mathcal{W}_{\xxx}(\mathbb{P},\mathbb{P}^N)\big]\big| > 
    C_2(\varphi_m,\xxx,m,N) + \varepsilon
\end{equation*}
occurs with probability at most $2\exp\big(-\frac{2N\varepsilon^2}{\tau(\varphi_m)^2\,\mathsf{d}(\xxx)^2}\big)$, where
\begin{equation*} 
    \begin{split}
        C_1(\varphi_m,\xxx,m,N) &\eqdef \tilde{C}_m\tau(\varphi_m)\,\mathsf{d}(\xxx)r_m(N)\\
        C_2(\varphi_m,\xxx,m,N) &\eqdef \tilde{C}_m(\tau(\varphi_m)-1)\mathsf{d}(\xxx)r_m(N).
    \end{split}
\end{equation*}
Here, the values of $r_m(N)$ and $\tilde{C}_m$ are provided in Table~\ref{tab:concentration_main_RATES}. 
Furthermore, the values of $\tau(\varphi_m)$ are not larger than what is recorded in Table~\ref{tab:concentration_main_RATES}, and in the case\footnote{We emphasize that the quantities $r_m(N)$ and $\tilde{C}_m$ remain as in Table~\ref{tab:concentration_main_RATES} even in the Euclidean setting case.} there exists $d\in\mathbb{N}$ such that $\xxx$ is a metric subspace of $\mathbb{R}^d$, they are bounded by what is recorded in Table~\ref{tab:Euclidean_Distortion}.
\end{theorem}

The conclusion (ii) suggests that $\mathcal{W}_{\xxx}(\mathbb{P},\mathbb{P}^N)$ concentrates around its expected value, with an additive term influenced by the distortion factor $\tau(\varphi_m)$. This is to be expected, as this term reflects the geometric constraints imposed by the bi-Lipschitz map $\varphi_m$ in representing $\xxx$ in an $m$-dimensional Euclidean space. 
The existence of such a bi-Lipschitz Euclidean embedding for any $m\in\mathbb{N}$, in turn, is guaranteed by Proposition~\ref{prop:EuclideanRep}.

{In Tables~\ref{tab:concentration_main_RATES} and~\ref{tab:Euclidean_Distortion}}, $\tilde{\varepsilon}_{m,k}$ is defined in \eqref{eqdef:tildeeps} below and that the worst-case values of $\tau(\varphi_m)$ are derived from Proposition~\ref{prop:EuclideanRep}.

\begin{table}[H]%
    \centering
    \caption{Concentration rates, dimension constants and worst-case distortion bounds} 
    \resizebox{\columnwidth}{!}{
    \label{tab:concentration_main_RATES}
        \begin{tabular}{@{}llll@{}}
        \toprule
            $\boldsymbol{m}$ & $\boldsymbol{r_m(N)}$ & $\boldsymbol{\tilde{C}_m}$ & 
            \textbf{Worst-case} $\boldsymbol{\tau(\varphi_m)}$\\
		  \midrule
		  $m=1$ & 
            $N^{-1/2}$ & $\frac{1}{\sqrt{8}-2}$ & $12k$
        \\~\\
		  $m=2$ & 
            $\big(32 + \log_2(N)\big)N^{-1/2}$ &         
            $\frac{1}{\sqrt{8}}$ & $12k$ 
        \\~\\
		  $3 \le m \le \lceil 8\log (k)\rceil $  
            & $N^{-1/m}$ &     
            $2\Big(\frac{\frac{m}{2} - 1}{2 (1-2^{1-m/2})}\Big)^{2/m}\Big(1 + \frac{1}{2(\frac{m}{2} - 1)}\Big)\,m^{1/2}$ 
            & 
            $720\, k^{2/m} \lfloor 2\log (k)+1\rfloor \big(\frac{\log (k)}{m}\big)^{1/2} \tilde{\varepsilon}_{m,k}$
        \\~\\
            $\lceil 8\log (k)\rceil < m < 2^k$ 
            & $N^{-1/m}$ & $2\Big(\frac{\frac{m}{2} - 1}{2 (1-2^{1-m/2})}\Big)^{2/m}\Big(1 + \frac{1}{2(\frac{m}{2} - 1)}\Big)\,m^{1/2}$
            & $48 \lfloor 2\log (k) +1\rfloor \,\tilde{\varepsilon}_{m,k}$
        \\~\\
            $2^k \le m$ 
            & $N^{-1/m}$ & $2\Big(\frac{\frac{m}{2} - 1}{2 (1-2^{1-m/2})}\Big)^{2/m}\Big(1 + \frac{1}{2(\frac{m}{2} - 1)}\Big)\,m^{1/2}$
            & $48\lfloor 2\log (k)+1\rfloor$\\ 
        \bottomrule
        \end{tabular} 
        }
\end{table}

\begin{table}[H]%
    \centering
    \caption{Worst-case distortion bound $\tau(\varphi_m)$ when $\xxx\subset \mathbb{R}^d$} 
    \label{tab:Euclidean_Distortion}
        \begin{tabular}{@{}lll@{}}
        \toprule
        $\boldsymbol{m}$ & $\boldsymbol{m < d}$ &          \textbf{Worst-case} $\boldsymbol{\tau(\varphi_m)}$ \\
		\midrule
		$m=1,2$ & Yes & $12k$\\
		$3 \le m \le \lceil 8\log (k)\rceil $ 
        & Yes 
        & $15\,k^{2/m}\big(\frac{\log (k)}{m}\big)^{1/2}$ \\
        $\lceil 8 \log (k)\rceil < m$ 
        & 
        Yes &
        $48 \lfloor 2\log (k) +1\rfloor \,\tilde{\varepsilon}_{m,k}$ \\ 
        $m\ge d$ &
        No &
        $1$\\ 
        \bottomrule
        \end{tabular} 
\end{table}

\begin{remark}
The key quantities in Table~\ref{tab:concentration_main_RATES} are, $r_m(N)$, which quantifies the optimal worst-case Wasserstein concentration rate, and the worst-case value of $\tau(\varphi_m)$, which quantifies the worst-case distortion bound incurred when mapping $\xxx$ to $\mathbb{R}^m$. 
The representation dimension constant $\tilde{C}_m$ carries no physical significance but is instead associated with the measure concentration rate in $\mathbb{R}^m$; see \cite{Kloeckner_2020_CounterCurse}. 
Observable from the table is a tension between the concentration rate $r_m(N)$ and the representation dimension $m$, manifested through $\tilde{C}_m$ and the worst-case $\tau(\varphi_m)$. 
This is because, as a function of $N$, $r_m(N)$ goes to $0$ faster for small $m$ than for large $m$. However, the worst-case $\tau(\varphi_m)$ is exponentially larger when $m$ is small than when $m$ is large, due to the fact that arbitrary finite geometries of $\xxx$ 
tend not to embed efficiently in small Euclidean dimensions. Therefore, for a fixed sample size $N$, the representation dimensions $m$ yielding the most minimal bounds are the ones that balance the value of $r_m(N)$ against those of $\tilde{C}_m$ and the worst-case $\tau(\varphi_m)$. 
\end{remark}

\subsection{A result on adaptive generalization and {reconstruction} bounds}
\label{s:Main__ss:Gen}

We begin by formalizing the setting introduced in Section~\ref{s:Introduction}. 
Let $(\xxx,d_{\xxx})$, $(\yyy,d_{\yyy})$ be two finite metric spaces and $(\xxx\times\yyy, d_{\xxx\times\yyy})$ be their product metric space \eqref{prodmetric}. 
Let $f^{\star}:\xxx\rightarrow \yyy$ be an unknown Lipschitz target function. 
Our objective is to derive statistical guarantees on the 
performance of models for $f^{\star}$ taken from a given hypothesis class $\mathcal{F}\subset\yyy^{\xxx}$. We specify $\mathcal{F}$ to be the following collection of {\it $L$-Lipschitz} hypotheses,
\begin{equation} 
\label{eqdef:compatible}
    \mathcal{F}_{L} \eqdef \{\hat{f}\in \yyy^{\xxx}:{\rm L}_{\mathit{u}}(\hat{f})\leq L\},
\end{equation}
for some $L\geq 0$. 
Let us be supplied with a finite number of i.i.d. training samples 
\begin{equation} \label{eq:datagenerating}
    (X_1,Y_1), \dots , (X_N,Y_N)\sim\mathbb{P},
\end{equation}
where $\mathbb{P}\in\mathcal{P}(\xxx\times\yyy)$ is a data generating probability measure. 
We stipulate that the training samples incorporate noise, which originates exclusively from the output components $Y_n$. Hence, it is expressed in \eqref{eq:datagenerating} that $\mathbb{P}$ encodes two key pieces of information. First, it captures the potential noise that obscures the relationship between inputs $X_n$ and outputs $Y_n$, which would be $f^{\star}(X_n)$ without such noise. 
Second, it accounts for the probability governing the sampling of any point $X\in\xxx$. To decouple these two pieces of information, we impose that the likelihood of sampling points in $\xxx$ is characterizable by a \textit{sampling probability measure} $\mu_{\xxx}\in\mathcal{P}(\xxx)$. With this, we define a probability measure $\mu\in\mathcal{P}(\xxx\times\yyy)$ such that
\begin{equation} \label{eqdef:mu}
    \mu \eqdef (\mathrm{Id}_{\xxx}\times f^{\star})_{\texttt{\#}}\mu_{\xxx},
\end{equation}
where $\mathrm{Id}_{\xxx}$ denotes the identity map on $\xxx$.
Thus, $\mu$ is a joint distribution quantifying the chance of sampling a random point $(X,f^{\star}(X))$ on the \textit{graph} of $f^{\star}$,
where $X$ is drawn according to $\mu_{\xxx}$.
Then the noise that masks the outputs of $f^{\star}$ can be examined through the total variation between $\mathbb{P}$, $\mu$, which is
\begin{equation} \label{eq:TVfinite}
   \mathrm{TV}(\mathbb{P},\mu) = \max_{A \in\{0,1\}^{\xxx}} \big|\mathbb{P}(A)-\mu(A)\big|.
\end{equation}
Considering \eqref{eq:datagenerating}, \eqref{eqdef:mu}, $\mathrm{TV}(\mathbb{P},\mu)$ intuitively functions as {a flexible} estimate of the {\textit{noise level}, capable of accommodating a wide range of noise patterns (for a standard interpretation as an \textit{additive noise}}, a concept rooted in classical {nonparametric} statistics (see Appendix~\ref{a:FurtherDiscussion__ss:GeneralizedAdditiveNoise}){).} {The connection between $f^{\star}$ and the data generating measure\footnote{{While $f^{\star}$ is generally non-unique, its introduction provides a fixed reference for conceptually distinguishing the underlying signal from stochastic corruption.}} $\mathbb{P}$ is controlled by a fixed noise level $0\le \Delta\le 2$, through}
\begin{equation}
\label{eq:SNR}
   \mathrm{TV}(\mathbb{P},\mu)\leq\Delta.
\end{equation}
In other words, $\mathbb{P}$ is the joint distribution of a random sample $(X,Y)$ lying near the graph of $f^{\star}$ with a proximity determined by $\Delta$, 
where $X$ is drawn according to $\mu_{\xxx}$.
Finally, we recall from Section~\ref{s:Introduction} the notions of empirical risk $\hat{\mathcal{R}}(\hat{f})$, the {population} risk $\mathcal{R}(\hat{f})$, defined in \eqref{eqdef:trurisk}, and the excess risk $\mathcal{R}^{\star}(\hat{f})$, defined in \eqref{eqdef:risks}, all of which are evaluated using a Lipschitz loss function $\mathcal{L}: \yyy \times \yyy \rightarrow [0,\infty)$.

Theorem~\ref{thm:MAIN_no_partition} below derives upper bounds for the worst-case generalization gap $\sup_{\hat{f}\in\mathcal{F}_{L}}\big|\mathcal{R}(\hat{f}) - \hat{\mathcal{R}}(\hat{f})\big|$ and the worst-case {reconstruction} gap $\sup_{\hat{f}\in\mathcal{F}_{L}} \big|\mathcal{R}^{\star}(\hat{f}) - \hat{\mathcal{R}}(\hat{f})\big|$, expressed in terms of the noise level $\Delta$ \eqref{eq:SNR}, the Lipschitz bound $L$, and the geometry of $\xxx\times\yyy$, which is accessed via a bi-Lipschitz Euclidean embedding. 
A proof is provided in Section~\ref{sec:proof_MAIN_no_partition}.

\begin{theorem}[Adaptive generalization and {reconstruction} bounds between finite metric spaces] 
\label{thm:MAIN_no_partition}
\hfill\\
Let $(\xxx,d_{\xxx})$, $(\yyy,d_{\yyy})$ be finite metric spaces and $(\xxx\times\yyy,d_{\xxx\times\yyy})$ be their product metric space, with $\mathrm{card}(\xxx\times\yyy)=k$. Let $f^{\star}: \xxx\to\yyy$ be a target Lipschitz function. Let 
$\mathcal{F}_{L}\subset\yyy^{\xxx}$ be the hypothesis class defined in \eqref{eqdef:compatible}. 
Let $\mathcal{L}:\yyy\times \yyy\rightarrow [0,\infty)$ be a Lipschitz loss function. 
Let $\mathbb{P}\in\mathcal{P}(\xxx\times\yyy)$ {satisfy} the bounded-noise assumption~\eqref{eq:SNR} for a fixed $0\le \Delta \le 2$. 
Let $(X_1,Y_1),\dots,(X_N,Y_N)\sim\mathbb{P}$ be i.i.d. random variables. 
Let $m\in \mathbb{N}$ and $\varphi_m:\xxx\times \yyy\to \mathbb{R}^m$ be a bi-Lipschitz embedding. Then for every~$\delta\in (0,1)$, both of the following events hold simultaneously with probability at least $1-\delta$: \\~\\
{\rm (i)} The worst-case generalization gap $\sup_{\hat{f}\in\mathcal{F}_{L}}\big|\mathcal{R}(\hat{f}) - \hat{\mathcal{R}}(\hat{f})\big|$ is at most
\begin{equation*}
    \bar{L}\mathsf{d}(\xxx\times\yyy)
    \bigg(\tilde{C}_m (2\tau(\varphi_m)-1)r_m(N)
    + \frac{\tau(\varphi_m)\sqrt{\log(2/\delta)}}{\sqrt{2N}}\bigg);
\end{equation*}
{\rm (ii)} The worst-case {reconstruction} gap $\sup_{\hat{f}\in\mathcal{F}_{L}} \big|\mathcal{R}^{\star}(\hat{f}) - \hat{\mathcal{R}}(\hat{f})\big|$ is at most
\begin{equation} \label{eq:noiseeffect}
    \bar{L}\mathsf{d}(\xxx\times\yyy)
    \bigg(\Delta + \tilde{C}_m (2\tau(\varphi_m)-1)r_m(N)
    + \frac{\tau(\varphi_m)\sqrt{\log(2/\delta)}}{\sqrt{2N}}\bigg).
\end{equation}
Here, $\bar{L}\eqdef{\rm L}_{\mathit{u}}(\mathcal{L})\max\{1,L\}$, and the values of $r_m(N)$ and $\tilde{C}_m$ are provided in Table~\ref{tab:concentration_main_RATES}. 
Furthermore, the values of $\tau(\varphi_m)$ are not larger than what is recorded in Table~\ref{tab:concentration_main_RATES}, and in the case there exists $d\in\mathbb{N}$ such that $\xxx$ is a metric subspace of $\mathbb{R}^d$, they are bounded by what is recorded in Table~\ref{tab:Euclidean_Distortion}.
\end{theorem}

\begin{remark}
{For clarity, we note the distinction between the two parts of Theorem~\ref{thm:MAIN_no_partition}. While Theorem~\ref{thm:MAIN_no_partition} (i) provides a standard generalization/estimation bound---since both the population risk $\mathcal{R}(\hat{f})$ and the empirical risk $\hat{\mathcal{R}}(\hat{f})$ are defined with respect to the same noisy distribution (see \eqref{eq:noisysamples} and \eqref{eqdef:trurisk})---Theorem~\ref{thm:MAIN_no_partition} (ii) establishes a reconstruction-type guarantee. That is, it quantifies the gap between the empirical risk (based on noisy data) and the excess risk (with respect to the clean distribution). When label noise introduces an irreducible discrepancy $\Delta>0$, this reconstruction gap is not expected to vanish as $N\to\infty$, as reflected in \eqref{eq:noiseeffect}. This is not an inconsistency, but rather an inherent phenomenon.}
\end{remark}

\section{Ramifications on statistical learning}
\label{s:Discussion}
{To place Theorem~\ref{thm:MAIN_no_partition} in context, we discuss its implications for statistical machine learning and compare them with classical generalization bounds for regression on continuous domains and binary classification on finite spaces.}
We focus on scenarios without {noise}, in which case, $\Delta=0$, and $\mathcal{R}^{\star}(\hat{f})=\mathcal{R}(\hat{f})$. Specifically, in Section~\ref{s:Discussion__ss:No_COD}, we compare our results against the \textit{Rademacher-complexity}-type generalization bounds for regression analysis. We demonstrate that in the large sample regime, our generalization bounds consistently overcome the curse of dimensionality, thanks to the constraints imposed by digital computing. 
Notably, this feature is absent in the Rademacher-type bounds. 
{In Section~\ref{s:Discussion_PAC}, we compare our results with PAC-generalization bounds for classifiers on finite spaces derived from VC-theoretic arguments. In this setting, Theorem~\ref{thm:MAIN_no_partition} yields competitive bounds across both practically large and small sample sizes, and remains robust as machine precision or model graphical dimension increases.}

\subsection{Digital computing softens the curse of dimensionality} \label{s:Discussion__ss:No_COD}

One major problem in the statistical learning theory for regression analysis is that many existing generalization bounds are adversely affected by the curse of dimensionality. A prototypical example concerns the generalization bounds for the class of $[0,1]$-valued, $L$-Lipschitz continuous hypotheses on $[0,1]^d$. Denote such class as $\mathrm{Lip}([0,1]^d;[0,1];L)$. In the case of no compromised sampled values, the generalization bounds based on Rademacher complexity \cite[Definition 2]{BartlettMendelson_2002}, as derived from \cite[Theorem 8]{BartlettMendelson_2002} 
and detailed in~\cite[Lemma 25]{hou2022instance}, ensure that for any $\delta \in (0,1]$, 
\begin{equation} \label{eq:Rademacher_term}
    \begin{split}
        \sup_{\hat{f}\in\mathrm{Lip}([0,1]^d;[0,1];L)} \big| 
        \mathcal{R}(\hat{f}) - \hat{\mathcal{R}}(\hat{f})\big| 
        \le
        &\frac{C_{d,L}}{N^{1/(d+3)}}
        + \frac{C'\sqrt{\log (2/\delta)}}{\sqrt{N}}
    \end{split}  
\end{equation}
holds with probability at least $1-\delta$.
Here, $C_{d,L}>0$ depends on $d, L$, as well as the uniform norm of the loss function $\mathcal{L}$, while $C'>0$ is an {absolute} constant; both are explicitly given in \cite[Lemma 25]{hou2022instance}.
Observe that the upper bound in \eqref{eq:Rademacher_term} is hampered by dimension $d$, which generally resists improvement through localization techniques, e.g. \cite{BartlettBousquetMendelson_LocRadCompl_AnnStat_2005} or \cite{hou2022instance}.
In contrast, we demonstrate below that the curse of dimensionality fades when taking digital computing constraints \eqref{eq:Rdpm} into account. Precisely, we consider the discretized setting with $\xxx=\mathbb{R}_{p,M}^d\cap [0,1]^d$, $\yyy=\mathbb{R}_{p,M}^1\cap [0,1]$, where $p$, $M$ are fixed, and the corresponding discretized class $\mathcal{F}_{L+1}$ of $(L+1)$-Lipschitz hypotheses, mapping from $\xxx$ to $\yyy$.
The lowered Lipschitz regularity, from $L$ to $L+1$, is a side effect of discretization. We illustrate this in Lemma~\ref{lem:RoundingLemma_Abstract} in Appendix~\ref{s:Lipschitz_Bounds}. Then following conclusion (i) of Theorem~\ref{thm:MAIN_no_partition}, when the representation dimension of $\xxx\times\yyy$ is $m=1$, it holds for every $\delta \in (0,1)$, 
\begin{equation}
\label{eq:our_bound}
    \sup_{\hat{f}\in \mathcal{F}_{L+1}} \big| \mathcal{R}(\hat{f}) - \hat{\mathcal{R}}(\hat{f})\big|
    \le
    2\bar{L} \mathsf{d}(\xxx\times\yyy)\Big(\tilde{C}_1\frac{c'_k}{\sqrt{N}}
    + \frac{C'_k\sqrt{\log(2/\delta)}}{\sqrt{N}}\Big),
\end{equation}
with probability at least $1-\delta$, where, recall that $\bar{L}={\rm L}_{\mathit{u}}(\mathcal{L})\max\{1,L\}$ and that $k$ represents ${\rm card}(\xxx\times\yyy)$. 
Therefore, unlike the Rademacher bound in \eqref{eq:Rademacher_term}, which {applies in the absence of discretization} and converges at a rate of $\mathcal{O}(1/N^{1/(d+3)})$, the bound in \eqref{eq:our_bound} converges at a rate of $\mathcal{O}(1/N^{1/2})$ in terms of the sample size $N$, independent of the ambient dimension $d$.
{We should note that the constants $c'_k$, $C'_k$ in \eqref{eq:our_bound} can be inferred from the worst-case distortion bounds in Table~\ref{tab:Euclidean_Distortion} to depend \textit{linearly} on $k$. Thus, the bound in \eqref{eq:our_bound} is most relevant in the regime $N\gg k$ (sampling with replacement), where $k$ can be treated as effectively constant relative to $N$.
(For an alternative application of our theory that provides a strategy to mitigate the effect of large $k$, when $N$ is only moderately large, we refer the reader to the {following} Section~\ref{s:Discussion_PAC}.)}

\subsection{Comparison with PAC-learning guarantees} 
\label{s:Discussion_PAC}

\paragraph{{Comparison with VC-type bounds under discretization refinement}}
Let ${\eta}\in (0,1)$. 
Recall that a ${\eta}$-packing of $[0,1]^d$ is a \textit{maximal collection} of points in $[0,1]^d$ such that any two points are at least ${\eta}$ apart, in the sense that adding any point $x\in [0,1]^d\setminus S$ would violate this separation property.  
In learning theory, the cardinality of a ${\eta}$-packing serves as a key measure for quantifying the complexity of a set via its size.  
This cardinality $k$ of the largest ${\eta}$-packing of $[0,1]^d$ is called the {\it ${\eta}$-packing number} of $[0,1]^d$; see \cite[Definitions~14,~16]{galimberti2022designing}. We recall from \citep[Chapter~15, Proposition 1.3]{lorentz1996constructive} that such a number $k$ satisfies
\begin{equation} \label{k}
    {2^{-d}} \,\big(\sqrt{d}/{\eta})^d\leq k\leq 
    {3^d} \big(\sqrt{d}/{\eta} \big)^d.
\end{equation}
Let $k\in\mathbb{N}$ be adhering to \eqref{k}, and let $\xxx\subset [0,1]^d$ be maximal $k$-point packing set, which means 
\begin{equation} \label{packingdistance}
    \sqrt{d}/(2k^{1/d})\leq{\eta}\leq \|x-z\|_2 \leq \sqrt{d}
\end{equation}
for every $x,z\in\xxx$. We consider the class of mappings that map $(\xxx,\|\cdot\|_2)$ to $(\{0,1\},|\cdot|)$. Let $\hat{f}$ be one such mapping, i.e. $\hat{f}\in\{0,1\}^{\xxx}$. 
Then from \eqref{packingdistance},
\begin{equation} \label{eq:WC_Lip_Constant__all_of_X}
    {\rm L}_{\mathit{u}}(\hat{f})
    \le 
    \frac{\max_{y,u\in \{0,1\}} |y-u|}{\min_{x,z\in \xxx;\,x\neq z} \|x-z\|_2}
    \le \frac{2k^{1/d}}{\sqrt{d}}. 
\end{equation}
By using the estimate \eqref{eq:WC_Lip_Constant__all_of_X}, we can leverage Theorem~\ref{thm:MAIN_no_partition} to uniformly bound the worst-case generalization gap over the entire class of binary classifiers on $(\xxx,\|\cdot\|_2)$.  While our result is designed for hypothesis subclasses with pre-specified bounded Lipschitz constants, this broader applicability is an immediate corollary. 
We {contrast} the derived generalization bound against the VC-theoretic Occam's Razor bound \citep[Corollary 4.6]{shalev2014understanding}, formulated as 
\begin{equation} \label{Occam}
    \sup_{\hat{f}\in \{0,1\}^{\xxx}} \big|\mathcal{R}(\hat{f}) - \hat{\mathcal{R}}(\hat{f})\big|
    \le 
    \frac{\sqrt{\log(2/\delta) + k\log(2)}}{\sqrt{2N}}.
\end{equation}
{(We have used that the relevant set of classifiers is $\{0,1\}^{\mathscr{X}}$, whose cardinality is $2^k$, given that $\mathscr{X}$ is a $k$-point set.)}
{When the machine precision becomes arbitrarily accurate; i.e. if $p\to\infty$ in \eqref{eq:Rdpm}, the Occam's razor bound cannot be controlled uniformly, due to the term $\sqrt{k\log(2)}$ in~\eqref{Occam}.
In contrast, the generalization bound derived via our theory provides greater stability.
Particularly, in the case of high dimensions, 
we consider the special case of Theorem~\ref{thm:MAIN_no_partition} with the representation dimension $m= d$.
Using Table~\ref{tab:Euclidean_Distortion}, we choose $\varphi_m=\varphi_d$ as the identity mapping on $\mathbb{R}^m$ to obtain $\tau(\varphi_m)=1$.
Then Theorem~\ref{thm:MAIN_no_partition}, along with \eqref{k}, \eqref{packingdistance}, \eqref{eq:WC_Lip_Constant__all_of_X}, implies that for each $\delta \in (0,1)$,}
\begin{equation} \label{eq:ourRademacher}
    {\sup_{\hat{f}\in \{0,1\}^{\xxx}} }
    \big |\mathcal{R}(\hat{f})-\hat{\mathcal{R}}(\hat{f})\big|
    \le 
     \frac{2k^{1/d}}{{\sqrt{d}}} \Big(\frac{4\sqrt{d}}{N^{1/d}} +\frac{\sqrt{\log(2/\delta)}}{\sqrt{N}}\Big) {\leq \frac{6}{\eta} \Big(\frac{4\sqrt{d}}{N^{1/d}} +\frac{\sqrt{\log(2/\delta)}}{\sqrt{N}}\Big)}
\end{equation}
{holds with probability at least $1-\delta$}. Here, we have used the fact that $\tilde{C}_m=\tilde{C}_d\le 4\sqrt{d}$ for {$m=d>2$}. 
{Thus, for fixed $\delta\in (0,1)$, \eqref{eq:ourRademacher} exhibits greater uniform robustness than \eqref{Occam}, which remains scaled explicitly with $k$.}

{Consequently, our worst-case generalization bound automatically selects between a fast-converging, cardinality-dependent rate (as in \eqref{eq:our_bound}) and a slower-converging, cardinality-stable rate (as in \eqref{eq:ourRademacher}). By allowing the full range of embedding dimensions as in Theorem~\ref{thm:MAIN_no_partition}(i), this adaptivity captures additional nuances.}

\paragraph{{Comparison with Rademacher-type bounds for Lipschitz hypotheses}}

{We offer an example that incorporates the Rademacher complexity of Lipschitz hypotheses in a graph-based setting.
Let $G=(\mathscr{X},E)$ be a simple, unweighted graph on a $k$-point vertex set $\mathscr{X}$. The graph $G$ naturally induces a finite metric space $(G,d_G)$, where $d_G$ denotes the \textit{shortest path metric} $d_G$. 
For $x\in\mathscr{X}$ and $r\geq 0$, let $B(x,r) \eqdef \{z\in\mathscr{X}: d_G(x,z)\leq r\}$ be a \textit{(closed) ball} of radius $r$ in this metric space. Then the \textit{(metric) doubling constant} of $(G,d_G)$ is the \textit{smallest} integer $\mathsf{K}_G\in\mathbb{N}$ such that every closed ball of radius $r$ can be covered by at most $\mathsf{K}_G$ closed balls of radius $r/2$. 
Now fix $L>0$. By~\cite[Theorem 4.3]{GottliebAryehKrauthgamer2016TCS}, the Rademacher complexity of the class of $L$-Lipschitz functions $f:(G,d_G)\to[-1,1]$---denoted $\operatorname{Lip}(G;[-1,1];L)$---with respect to a sample $(x_1,\dots,x_N)$, where $x_i\in\mathscr{X}$, is at most $CLN^{-1/(1+\log_2(\mathsf{K}_G))}$, for some absolute constant $C>0$.
Subsequently, applying \cite[Theorem 5]{bartlett2002rademacher}, we obtain with probability at least $1-\delta$ that  
\begin{equation} 
\label{RademacherFail1}
    \sup_{\hat{f}\in \operatorname{Lip}(G;[-1,1];L)} \big|\mathcal{R}(\hat{f}) - \hat{\mathcal{R}}(\hat{f})\big| 
    \leq \frac{CL}{N^{1/(1+\log_2(\mathsf{K}_G))}} + \frac{\sqrt{\log(1/\delta)}}{\sqrt{N}}.
\end{equation} 
(One can observe the similarity between \eqref{RademacherFail1} and \eqref{eq:Rademacher_term}, except that here we are dealing with a finite $L$-Lipschitz hypothesis class.) 
For $k\geq 2$, we trivially have $\mathsf{K}_G\leq k$, in the absence of further information on $G$. Thus $\log_2(\mathsf{K}_G)\leq \log_2(k)$. Substituting this into \eqref{RademacherFail1} yields:
\begin{equation} 
\label{RademacherFail2}
    \sup_{\hat{f}\in \operatorname{Lip}(G;[-1,1];L)} \big|\mathcal{R}(\hat{f}) - \hat{\mathcal{R}}(\hat{f})\big| 
    \leq \frac{CL}{N^{1/(1+\log_2(k))}} + \frac{\sqrt{\log(1/\delta)}}{\sqrt{N}}.
\end{equation} 
{Consequently, as in the example of Section~\ref{s:Discussion__ss:No_COD}, for sufficiently large graphs with bounded diameter, the convergence rate in \eqref{RademacherFail2} is strictly slower than the optimal asymptotic rate $\mathcal{O}(1/\sqrt{N})$ guaranteed by Theorem~\ref{thm:MAIN_no_partition}, particularly when $N\in\Omega(k^2)$.}



\section{Application to deep learning on digital computers}
\label{s:Applications}

We apply Theorem~\ref{thm:MAIN_no_partition} to examine the impact of digital computing constraints on the statistical learning theory of ReLU deep neural networks, as detailed in Corollary~\ref{cor:MLPDiscretetization} below, {\textit{in the absence of noise}}.
As for the noisy variation, Theorem~\ref{thm:MAIN_no_partition} (ii) ensures that our upper bounds on the worst-case {reconstruction} gap differ from those on the worst-case generalization gap by an additive factor of $\bar{L}\mathsf{d}(\xxx\times\yyy)\Delta$, where $\Delta$ is the noise level defined in~\eqref{eq:SNR}.

We begin by describing our selection process for the target and hypothesis classes as well as the underlying rationale in Section~\ref{sec:appsetting}. 
Then we present the corollary statement in Section~\ref{sec:appresult}, along with a note in Remark~\ref{rem:significance} on the implicated immunity to the curse of dimensionality in learning with digital computers. 



\subsection{Setting} \label{sec:appsetting}

\paragraph{Input and output spaces} 

We consider {the grid discretization $\mathbb{R}_{p,M}^d$ of $\mathbb{R}^d$ \eqref{eq:Rdpm}} introduced in Section~\ref{s:Introduction}, {equipped with the Euclidean distance}. Here, $p, M$ are implicitly determined by the specific digital hardware used to implement the learning problem.
{In the compressed sensing~\cite{choi2021sparse,gross2024sparse} and sparse grids literature~\cite{pfluger2010spatially}, one circumvents the curse of dimensionality by focusing on function classes of low-complexity structures that are learnable from a small number of samples. Motivated by this, we adopt the following assumption\footnote{{
Assumption~\ref{ass:sparse_support} also aligns with common structural assumptions in learning theory. The intent is to restrict the geometry of the sampling distribution without imposing additional smoothness on the target function.}} for $\mu_{\xxx}\in \mathcal{P}(\mathbb{R}_{p,M}^d)$.}
{\begin{assumptions}[Sparse support for sampling measure]
\label{ass:sparse_support}
There exist constants $C\geq 1$, $s\geq 0$ such that 
\begin{equation*}
    {\rm card}({\rm supp}(\mu_{\xxx})) \leq Cd^s,
\end{equation*}
where ${\rm supp}(\mu_{\xxx})$ denotes the support of $\mu_{\xxx}$.
\end{assumptions}}

{Note that $\mathscr{X}={\rm supp}(\mu_{\xxx})$. We set $\mathscr{Y}=\mathbb{R}^1_{p,M}$.}

\paragraph{Idealized target functions}
From elementary statistical learning theory, we know that a function class is not PAC-learnable if its members are of arbitrary complexity. For instance, a set of classifiers is not PAC-learnable if the class has infinite VC dimension, see e.g.~\cite[Theorem 6.7]{shalev2014understanding}. Similar combinatorial complexity characterizations of PAC-learnability are available in the regression context, see e.g.~\cite{attias2023optimal}. Consequently, to avoid such pathologies, we restrict to idealized target functions belonging to a class of {\it bounded complexity}. 

The idealized target class we explore here comprises finite combinations of a piecewise linear Riesz basis in $L^2([0,1])$ that is particularly suited to deep learning theory. We start by describing this basis in dimension one, which consists of piecewise linear ``cosine" and ``sine" functions, defined as follows. 
Let $\mathcal{C},\mathcal{S}: [0,1]\to\mathbb{R}$ be such that
\begin{equation*}
    \mathcal{C}(t)
    \eqdef 
    \begin{cases}
        1 - 4t & \mbox{ if } t\in [0,1/2)\\
        4t - 3 & \mbox{ if } t \in [1/2,1]
    \end{cases},
\end{equation*}
and
\begin{equation*}
    \mathcal{S}(t)
    \eqdef 
    \begin{cases}
        4t & \mbox{ if } t\in [0,1/4)\\
        2 - 4t & \mbox{ if } t\in [1/4,3/4)\\
        4t - 4 & \mbox{ if } t \in [3/4,1]
    \end{cases}.
\end{equation*}
Next, for each $j\in\mathbb{N}$ we define the ``higher frequency" versions of these functions on $\mathbb{R}$ to be 
\begin{equation*}
    \mathcal{C}_j(t)
    \eqdef 
    \mathcal{C}(jt - \lfloor jt\rfloor )
    \quad\text{ and }\quad
    \mathcal{S}_j(t)
    \eqdef 
    \mathcal{S}(jt - \lfloor jt\rfloor).
\end{equation*}
By \cite[Proposition 6.1]{daubechies2022nonlinear}, $\{\mathcal{C}_j,\mathcal{S}_j\}_{j\in\mathbb{N}}$ constitutes a Riesz basis of $L^2([0,1])$. Fix $K,{\mathbf{M}}\in\mathbb{N}$, which quantify the number of frequencies and magnitude respectively. Let $\Lambda\subset\mathbb{N}$ be a dictionary with $\mathrm{card}(\Lambda)=K$. The idealized target class in dimension one is defined to be\footnote{The bounded complexity is reflected in the finite linear combinations indexed over a fixed finite set, with linear coefficients confined to a bounded interval.} 
\begin{equation} \label{eq:conceptclassdim1}
    \big\{ f = \sum_{\lambda \in \Lambda} 
    (a_{\lambda}\mathcal{C}_{\lambda} + b_{\lambda} \mathcal{S}_{\lambda}): |a_{\lambda}|,|b_{\lambda}|\le {\mathbf{M}} \big\}.
\end{equation}
A multivariate version of \eqref{eq:conceptclassdim1} is straightforward to define. This extension encompasses functions $f:\mathbb{R}^d\rightarrow\mathbb{R}$ of the following form
\begin{equation} \label{eq:conceptclass}
    f(x) = 
    \sum_{i=1}^d f_i((Ux)_i),
\end{equation}
where $(x)_i$ denotes the $i^{th}$ component of $x\in\mathbb{R}^d$, the matrix $U\in\mathbb{R}^{d\times d}$ belongs to the orthogonal group $O(d)$, and further, each $f_i$ is an element of \eqref{eq:conceptclassdim1}.

\paragraph{Function discretization}

The quantization of a mapping $f:\mathbb{R}^d\to\mathbb{R}$ in digital computers is achieved by restricting inputs to $\mathbb{R}^d_{p,M}$ and rounding the outputs of $f$ so that they fall within $\mathbb{R}^1_{p,M}$. 
Consequently, we will work with target functions mapping from $\mathbb{R}^d_{p,M}$ to $\mathbb{R}^1_{p,M}$ taking the form $x \mapsto \Pi\circ f(x{|_{\mathscr{X}}})$, where $f$ belongs to a specific idealized class of bounded complexity chosen above. 
Here, $\Pi:\mathbb{R}\to\mathbb{R}^1_{p,M}$ denotes a fixed choice of a nearest neighbor rounding operation, sending each $y\in\mathbb{R}$ to $u\in \mathbb{R}^1_{p,M}$ such that $\|u-y\|$ is minimized, i.e.
\begin{equation*} 
    \Pi(y) 
    \in
    \operatorname{argmin}
    _{u\in \mathbb{R}^1_{p,M}} \|u-y\|.
\end{equation*}
For instance, if $d=1$ and $p=0$, then $\Pi$ can correspond to either the rounding up operation (i.e.\ the integer ceiling) or the rounding down operation (i.e.\ the integer floor).


\paragraph{Target class and hypothesis class}

The aforementioned functional learnability is underlined by the implementability of \eqref{eq:conceptclassdim1} through a ReLU NN. 
We denote by $\mathfrak{L}_{\mathbf{W},\mathbf{L}, \mathbf{B}}$ the class of functions learnable by a ReLU NN with $\mathbf{W}\in\mathbb{N}$ width, $\mathbf{L}\in\mathbb{N}$ layers, and $\mathbf{B}$ bound in weight-bias complexity. Specifically, each $f\in \mathfrak{L}_{\mathbf{W},\mathbf{L},\mathbf{B}}$ admits the representation
\begin{alignat*}{2}
    &f(x)
    && = 
    A^{(\mathbf{L})}x^{(\mathbf{L})},\\
    &x^{(l+1)} && = \mathrm{ReLU}\bullet (A^{(l)} x^{(l)} + b^{(l)} )
    \quad \text{ for }\quad l=0,\dots,\mathbf{L}-1,\\
    &x^{(0)} &&= x.
\end{alignat*}
Here, $A^{(l)}\in\mathbb{R}^{d_{l+1}\times d_l}$ with $|A_{i,j}^{(l)}|,|b^{(l)}_j|\le \mathbf{B}$, $b^{(l)}\in\mathbb{R}^{d_{l+1}}$, and $d_l\le \mathbf{W}$, $d_0=d$, $d_{\mathbf{L}}=1$. As demonstrated in \citep[Theorem 6.2]{daubechies2022nonlinear}, for each $\mathbf{W}\geq 6$, 
\begin{equation} \label{eq:Daubechies}
    \big\{ f = \sum_{\lambda \in \Lambda} 
    (a_{\lambda}\mathcal{C}_{\lambda} + b_{\lambda} \mathcal{S}_{\lambda}): |a_{\lambda}|,|b_{\lambda}|\leq {\mathbf{M}} \big\} \subset \mathfrak{L}_{\mathbf{W},\mathbf{L}, \mathbf{B}},
\end{equation}
with
\begin{equation} \label{eq:LB}
    \mathbf{L} = 2(\lceil \Lambda^{\star}\rceil + 2)\bigg\lfloor \frac{K}{\lfloor \frac{\mathbf{W}-2}{4}\rfloor}\bigg\rfloor, \quad\text{ and }\quad
    \mathbf{B} = \max_{\lambda\in\Lambda} \{a_{\lambda},b_{\lambda},8\},
\end{equation}
where {$\Lambda^{\star} \eqdef \max_{\lambda \in \Lambda}\, \lambda$}.
Hence, it follows from \eqref{eq:Daubechies}, \eqref{eq:LB} that, for each $\mathbf{W}\geq 6$, each member $f$ of \eqref{eq:conceptclass} is realizable by a ReLU NN with $\mathbf{W}$ width, $\mathbf{L}+1$ layers, and $\mathbf{B}$ bound. Subsequently, we establish our target class $\mathcal{T}$ and hypothesis class $\mathcal{F}$ to be, respectively
\begin{equation} \label{eqdef:CFclass}
    \mathcal{T} =\Big\{\Pi\circ f{|_{\mathscr{X}}}: f \text{ belongs to } \eqref{eq:conceptclass} \Big\}
    \quad\text{ and }\quad
    \mathcal{F} = \Big\{ \Pi\circ f{|_{\mathscr{X}}}: f\in \mathfrak{L}_{6,\mathbf{L}+1,\mathbf{B}}\Big\},
\end{equation}
where $\mathbf{B}$ is as in \eqref{eq:LB} and $\mathbf{L} = 2K(\lceil \Lambda^{\star}\rceil + 2)$.

{We briefly remark that in this simplified setting, only the inputs and outputs are discretized, while internal layer computations are performed with high precision. Although this captures only a partial aspect of realistic digital computation, considering this hypothesis class already highlights the benefits of discretization for statistical learning, thus providing an initial insight into neural network generalization on digital computers.}

\subsection{Result} \label{sec:appresult}

The following corollary is a straightforward application of Theorem~\ref{thm:MAIN_no_partition}. 
A complete proof is presented in Appendix~\ref{s:Lipschitz_Bounds}.

\begin{corollary}[Generalization bounds for ReLU NNs on digital computers] \label{cor:MLPDiscretetization} \hfill\\
Let $\xxx={{\supp}(\mu_{\xxx})\subset\mathbb{R}^d_{p,M}}$ and $\yyy=\mathbb{R}^1_{p,M}$. Let the target class $\mathcal{T}$ and the hypothesis class $\mathcal{F}$ be as given in \eqref{eqdef:CFclass}. Let $\mathcal{L}:\yyy\times \yyy\rightarrow [0,\infty)$ be a $1$-Lipschitz loss function. Let $f^{\star}\in \mathcal{T}$ be a target function, and let the samples 
\begin{equation*}
    (X_1,f^{\star}(X_1)), \dots, (X_N,f^{\star}(X_N))
\end{equation*}
be given, where $X_n\sim \mathrm{Unif}(\mathbb{R}^d_{p,M})$. Then {under Assumption~\ref{ass:sparse_support}}, for every $\delta\in (0,1)$, it holds with probability at least $1-\delta$ that the worst-case generalization gap $\sup_{\hat{f}\in \mathcal{F}} \big|\mathcal{R}(\hat{f})-\hat{\mathcal{R}}(\hat{f})\big|$ is bounded above by a constant multiple of
\begin{equation}
\label{eq:ReLUGenBound}
    \frac{
        {d^s 2^{p+1}}
    \sqrt{d}\,
        {M^2}
    K\lceil \Lambda^{\star}\rceil\mathbf{B}\sqrt{\log(2/\delta)}}{\sqrt{N}}.
\end{equation}
\end{corollary}

\begin{remark} 
\label{rem:significance} 
Most available PAC-learnability guarantees for ReLU NNs on subsets of $\mathbb{R}^d$ derived from Rademacher complexity bound using optimal transport techniques converge at a rate of $\mathcal{O}(1/N^{1/(d+1)})$, thus suffering from the curse of dimension; see e.g.~\cite[Lemma 25 and Theorem 4]{hou2022instance}. 
Alternatively, as shown in~\cite[Theorem 1.1]{bartlett2017spectrally}, spectral/path-norm-type generalization bounds include an additive term of the form $\mathcal{O}(\|X\|_F/N)$, where $\|X\|_F$ is the Frobenius norm of the random matrix whose rows are the input sample data $X_1,\dots,X_N$. 
{For instance,} if $X_1,\dots,X_N$ are i.i.d., sub-Gaussian, and isotropic, then a version of Gordon's majorization theorem in~\cite[Theorem 4.6.1]{RomanhighDimensionalProbBook} implies that, for every $t>0$
\begin{equation} \label{eq:Gordon2}
\mathbb{P}\Big(\|X\|_F \ge \sqrt{\min\{N,d\}}\,
        (\sqrt{N}-c(\sqrt{d}+t))\Big)
        \ge 
        1-2e^{-t^2}.
\end{equation}
Consider the regime where the dimension grows with the sample size, taking $d=\sqrt{N}$ and $t=\sqrt{d} = N^{1/4}$. 
{In this case, \eqref{eq:Gordon2} implies that the additive Frobenius-norm term in the spectral bound of~\cite[Theorem 1.1]{bartlett2017spectrally} satisfies}
\begin{equation} 
\label{eq:LB_Spectral}
    \mathbb{P}\Big(\frac{\|X\|_F}{N} {\geq \frac{c}{N^{1/4}} } \Big)
    \ge 1-2e^{-\sqrt{N}}.
\end{equation}
By contrast, exploiting the sparse support of the sampling measure, our generalization bound \eqref{eq:ReLUGenBound} simplifies in this high-dimensional regime to (up to a log factor)
\begin{equation*} 
    2^{p+1}M^2 K\lceil \Lambda^\star \rceil \mathbf{B} N^{(s-1)/2},
\end{equation*}
which converges {at a faster rate which converges at a faster rate than that given by \eqref{eq:LB_Spectral} under sufficient sparsity, namely when $s\in [0,1/2)$.
}

Lastly, we note that existing generalization bounds for randomized ReLU MLPs, including PAC-Bayesian bounds~\cite{neyshabur2017pac,dziugaitecomputing}, achieve rates of $\mathcal{O}(1/\sqrt{N})$. These results, however, do not extend to our setting, as they apply only to networks with randomized weights and biases, whereas our focus is on standard, non-random ReLU networks.
\end{remark}

\section{Proofs} 
\label{sec:Proofs}

\subsection{Proof of Proposition~\ref{prop:EuclideanRep}} \label{sec:Proofs_Discrete_Atlas}

We organize the construction of the desired Euclidean embeddings into cases based on their dimensions, as follows:
\begin{itemize}
    \item $m=1,2$, the so-called \textit{ultra-low-dimensional} case,
    \item $2< m \le \lceil 8\log (k)\rceil$, the \textit{low-dimensional} case,
    \item $\lceil 8\log (k)\rceil < m \leq 2^k$, the \textit{high-dimensional} case,
    \item $m > 2^k$, the \textit{ultra-high-dimensional} case.
\end{itemize}
We navigate through each case with the help of the subsequent lemmas. 
In particular, Lemma~\ref{lem:ULD} addresses the first case, Lemma~\ref{lem:LD} the second case and Lemma~\ref{lem:HD} the third. Once these foundational results are established, we complete the proof of Proposition~\ref{prop:EuclideanRep}, which includes the fourth case.

\begin{lemma}[Ultra-low-dimensional metric embedding] \label{lem:ULD} \hfill\\ 
Let $(\xxx,d_{\xxx})$ be a finite metric space, with $\mathrm{card}(\xxx)=k$. Then there exist bi-Lipschitz embeddings $\varphi_1: \xxx\to \mathbb{R}$ and $\varphi_2: \xxx \to \mathbb{R}^2$ satisfying, for $x,\tilde{x}\in\xxx$,
\begin{equation} \label{embedding_dim1}
    d_{\xxx}(x,\tilde{x}) \le |\varphi_1(x)-\varphi_1(\tilde{x})|\le 12k\,d_{\xxx}(x,\tilde{x}),
\end{equation}
and 
\begin{equation} \label{embedding_dim2}
    d_{\xxx}(x,\tilde{x}) \le \|\varphi_2(x)-\varphi_2(\tilde{x})\|_2\le 12k\,d_{\xxx}(x,\tilde{x}).
\end{equation}
\end{lemma}

\begin{proof}
The conclusion \eqref{embedding_dim1} is a consequence of \citep[Theorem~2.1]{Matouvek_OptimalEuclidean_Npoint} and its proof. The upper bound is in the statement of the said theorem, while the lower bound follows from the fact that the constructed embedding map $\varphi_1$ is a \textit{non-contracting} map (see Statement~2.3 in \cite{Matouvek_OptimalEuclidean_Npoint}). 

In turn, the conclusion \eqref{embedding_dim1} implies \eqref{embedding_dim2}, as $(\mathbb{R},|\cdot|)$ can be isometrically embedded in $(\mathbb{R}^2,\|\cdot\|_2)$ via $x\mapsto (x,0)$. 
\end{proof}

\begin{lemma}[High-dimensional metric embedding] \label{lem:HD} \hfill\\
Let $(\xxx,d_{\xxx})$ be a finite metric space, with $\mathrm{card}(\xxx)=k$. Then for every $\lceil 8\log (k)\rceil <m\leq 2^k$ there exists a bi-Lipschitz embedding $\varphi_m: \xxx \rightarrow \mathbb{R}^m$ such that for $x,\tilde{x}\in\xxx$,
\begin{equation}\label{case2^k}
    \frac{d_{\xxx}(x,\tilde{x})}{48 \lfloor\log_2 k+1\rfloor} \le 
    \|\varphi_{2^k}(x)-\varphi_{2^k}(\tilde{x})\|_2 \le d_{\xxx}(x,\tilde{x}), 
\end{equation}
if $m=2^k$, and moreover,
\begin{align} \label{innercase}
    \frac{\sqrt{1-\varepsilon_{m,k}}}{48 \lfloor\log_2 k+1\rfloor} d_{\xxx}(x,\tilde{x}) 
    \leq 
    \|\varphi_m(x)-\varphi_m(\tilde{x})\|_2 
    \leq \sqrt{1+\varepsilon_{m,k}} \, d_{\xxx}(x,\tilde{x}),
\end{align}
if $\lceil 8\log (k)\rceil <m<2^k$, where 
\begin{equation} \label{eqdef:epsm}
    \varepsilon_{m,k} \eqdef 
    \frac{2\sqrt{2}\sqrt{\log (k)}}{\sqrt{m}}.
\end{equation}
\end{lemma}

\begin{proof}
The proof of Bourgain's Metric Embedding Theorem \cite{BourgainEmbedding_Original_1985}, as formulated in \citep[Theorem 15.7.1]{Matouvek_LecturesDiscreteGeometry_2002}, establishes the existence of a bi-Lipschitz embedding
$\phi_1: \xxx \rightarrow \mathbb{R}^{2^k}$ such that for $x,\tilde{x}\in\xxx$,
\begin{equation} \label{embedding_2^k}
    \frac{d_{\xxx}(x,\tilde{x})}{48 \lfloor\log_2 k+1\rfloor}\le \|\phi_1(x)-\phi_1(\tilde{x})\|_2 \le d_{\xxx}(x,\tilde{x}).
\end{equation}
Particularly, we note that the constant $(48\lfloor\log_2 k+1\rfloor)^{-1}$ in lower bound in \eqref{embedding_2^k} can be inferred by combining the proof of Theorem~15.7.1 and the statement of \citep[Lemma~15.7.2]{Matouvek_LecturesDiscreteGeometry_2002}. By letting $\varphi_{2^k}\eqdef\phi_1$, we obtain \eqref{case2^k} from \eqref{embedding_2^k}. 

For each other $m$ where $\lceil 8\log (k)\rceil<m<2^k$, we let $\varepsilon_{m,k}$ be as in \eqref{eqdef:epsm}. 
A version of the Johnson-Lindenstrauss lemma, as formulated in \citep[Theorem 2.1]{DubhashiPanconesi_2009_Concentration}, implies the existence of a mapping $\phi_2: \mathbb{R}^{2^k} \rightarrow \mathbb{R}^m$, such that if $u,\tilde{u} \in\mathbb{R}^{2^k}$, then 
\begin{align} \label{embedding_mid_range}
    \sqrt{1-\varepsilon_{m,k}}\,\|u-\tilde{u}\|_2 &\le \|\phi_2(u)-\phi_2(\tilde{u})\|_2 
    \le \sqrt{1+\varepsilon_{m,k}}\,\|u-\tilde{u}\|_2,
\end{align}
as long as
\begin{equation}\label{dim_condition}
    m \geq \frac{4\log (k)}{\varepsilon_{m,k}^2/2-\varepsilon_{m,k}^3/3}. 
\end{equation}
However, \eqref{dim_condition} is readily satisfied by the selection \eqref{eqdef:epsm}. Therefore \eqref{embedding_mid_range} holds. By letting $\varphi_m\eqdef \phi_2\circ\phi_1 = \phi_2\circ\varphi_{2^k}$ and combining \eqref{embedding_2^k}, \eqref{embedding_mid_range}, we obtain
\begin{align*}
    \frac{\sqrt{1-\varepsilon_{m,k}}}{48 \lfloor\log_2 k+1\rfloor}\,d_{\xxx}(x,\tilde{x}) &\le \|\varphi_m(x)-\varphi_m(\tilde{x})\|_2 
    \le \sqrt{1+\varepsilon_{m,k}} \,d_{\xxx}(x,\tilde{x}),
\end{align*}
for every $x,\tilde{x}\in\xxx$, which is \eqref{innercase}, as desired. 
\end{proof}

Note that the integer range from $\lceil 8\log (k)\rceil + 1$ to $2^k$ includes $m=k$. Hence, as a result of Lemma~\ref{lem:HD}, there exists a bi-Lipschitz embedding $\varphi_k: \xxx \to \mathbb{R}^k$ such that for every $x,\tilde{x}\in\xxx$,
\begin{align} \label{casem=k}
    \frac{\sqrt{1-\varepsilon_{k,k}}}{48 \lfloor\log_2 k+1\rfloor} d_{\xxx}(x,\tilde{x}) \le \|\varphi_k(x)-\varphi_k(\tilde{x})\|_2
    \le \sqrt{1 + \varepsilon_{k,k}}\,d_{\xxx}(x,\tilde{x}).
\end{align}
Therefore, we can assume that the points in $\xxx$ are already situated within the Euclidean space $(\mathbb{R}^k,\|\cdot\|_2)$. This enables us to apply the following known result. 

\begin{lemma}[Metric embedding in low-dimensional Euclidean space] \label{lem:LD_Euclidean} \hfill\\
Let $\xxx$ be an $k$-point subset of $(\mathbb{R}^k,\|\cdot\|_2)$. Then for every $3\leq m\leq \lfloor\log (k)\rfloor$, there exists a mapping $\phi_m: \mathbb{R}^k \to \mathbb{R}^m$ such that for $x,\tilde{x}\in\xxx$:
\begin{align} \label{eq:lowdimBourgain}
    \frac{\|x-\tilde{x}\|_2}{3k^{2/m}\sqrt{k/m}} 
    \leq \|\phi_m(x)-\phi_m(\tilde{x})\|_2 
    \leq 5\sqrt{\log (k)/k} \, \|x-\tilde{x}\|_2.
\end{align}
\end{lemma}

\begin{proof}
The verification of \eqref{eq:lowdimBourgain} can be derived from the proof of \citep[Theorem~2.2]{Matouvek_OptimalEuclidean_Npoint}. Notably, the Lipschitz constants featured in \eqref{eq:lowdimBourgain} are precisely the values $a,b$ in the said proof. 
\end{proof}

Drawing from the insights of Lemma~\ref{lem:LD_Euclidean}, we deduce the following.

\begin{lemma}[Low-dimensional metric embedding] \label{lem:LD} \hfill\\
Let $(\xxx,d_{\xxx})$ be a finite metric space with $\mathrm{card}(\xxx)=k$. Then for every $3\leq m\leq\lfloor\log (k)\rfloor$, there exists a bi-Lipschitz embedding $\varphi_m: \xxx \to \mathbb{R}^m$ such that, for every $x,\tilde{x}\in\xxx$,
\begin{align} \label{lowdimembedding}
    a_{m,k} \,d_{\xxx}(x,\tilde{x}) \leq \|\varphi_m(x)-\varphi_m(\tilde{x})\|_2 \leq b_k\,d_{\xxx}(x,\tilde{x}),
\end{align}
where 
\begin{align*}
    a_{m,k} \eqdef \frac{\sqrt{1-\varepsilon_{k,k}}}{144k^{2/m}\sqrt{k/m}\lfloor\log_2 k+1\rfloor}, \quad\text{ and }\quad 
    b_k \eqdef 5\sqrt{\log (k)/k}\sqrt{1+\varepsilon_{k,k}}.
\end{align*}
\end{lemma}

\begin{proof} For each $3\leq m\leq\lfloor \log (k)\rfloor$, let $\varphi_m \eqdef \phi_m\circ\varphi_k$, where $\varphi_k$ is as described in \eqref{casem=k} and $\phi_m$ in Lemma~\ref{lem:LD_Euclidean}. Then \eqref{lowdimembedding} readily follows from a combination of \eqref{casem=k} and \eqref{eq:lowdimBourgain}. 
\end{proof}

We are now in a position to derive a complete proof for Proposition~\ref{prop:EuclideanRep}, relying on the established lemmas. 

\begin{proof}[Proof of Proposition~\ref{prop:EuclideanRep}] 
For $m,k\in\mathbb{N}$, let
\begin{equation*} 
    \tilde{\varepsilon}_{m,k} \eqdef 
    \frac{\sqrt{1+\varepsilon_{m,k}}}{\sqrt{1-\varepsilon_{m,k}}}.
\end{equation*}
When $m=1,2$, let $\varphi_m$ be as in Lemma~\ref{lem:ULD}. Then \eqref{embedding_dim1}, \eqref{embedding_dim2} produce respectively that $\tau(\varphi_1), \tau(\varphi_2)\leq 12k$. When $3\leq m\leq\lfloor\log (k)\rfloor$, let $\varphi_m$ be as in Lemma~\ref{lem:LD}. Then it can be seen from \eqref{lowdimembedding} that
\begin{align} \label{ultralowtau}
    \tau(\varphi_m) \leq \frac{b_k}{a_{m,k}} 
    \leq 720\, k^{2/m}\lfloor 2\log (k)+1\rfloor \, \frac{\sqrt{\log (k)}}{\sqrt{m}} \, \frac{\sqrt{1+\varepsilon_{m,k}}}{\sqrt{1-\varepsilon_{m,k}}}
    = 720\, k^{2/m}\lfloor 2\log (k)+1\rfloor \, \frac{\sqrt{\log (k)}}{\sqrt{m}} \, \tilde{\varepsilon}_{m,k}.
\end{align}
When $\lfloor \log (k)\rfloor + 1\leq m\leq \lceil 8\log (k)\rceil$, we note that $(\mathbb{R}^{\lfloor\log (k)\rfloor},\|\cdot\|_2)$ can be isometrically embedded into $(\mathbb{R}^m,\|\cdot\|_2)$ via 
\begin{equation} \label{embed}
    (x_1,\dots,x_{\lfloor\log (k)\rfloor})\mapsto (x_1,\dots,x_{\lfloor\log (k)\rfloor},0,\dots,0). 
\end{equation}
Then by left composing this map \eqref{embed} with $\varphi_{\lfloor \log (k)\rfloor}$ defined in the previous case, we acquire a bi-Lipschitz embedding $\varphi_m: \xxx \to \mathbb{R}^m$, where $m$ {is in the said range}, for which \eqref{ultralowtau} continues to hold. When $\lceil 8\log (k)\rceil + 1\leq m\leq 2^k$, we let $\varphi_m$ be as in Lemma~\ref{lem:HD}, where it follows from \eqref{innercase}, \eqref{case2^k} respectively that, for $\lceil 8\log (k)\rceil + 1\leq m< 2^k$,
\begin{equation*}
    \tau(\varphi_m) \leq 48 \lfloor 2\log (k) +1\rfloor\, \frac{\sqrt{1+\varepsilon_{m,k}}}{\sqrt{1-\varepsilon_{m,k}}} = 48 \lfloor 2\log (k) +1\rfloor\, \tilde{\varepsilon}_{m,k},
\end{equation*}
and for $m=2^k$,
\begin{equation*}
    \tau(\varphi_{2^k}) 
    \leq 48\lfloor 2\log (k)+1\rfloor.
\end{equation*}
When $m>2^k$, we perform an isometric embedding of $(\mathbb{R}^{2^k},\|\cdot\|_2)$ into $(\mathbb{R}^m,\|\cdot\|_2)$. 
This addresses the remaining case of the theorem, and we conclude the proof.
\end{proof}

\subsection{Proof of Theorem~\ref{thrm:ConcentrationFiniteMetricSpaces}} \label{sec:propCOMproof}

We first present a crucial concentration result that forms the basis for Theorem~\ref{thrm:ConcentrationFiniteMetricSpaces}. This result, articulated as Lemma~\ref{lem:conc_Wassmetric} below, whose full version was given in \cite[Lemma~B.5]{hou2022instance}, studies the Wasserstein distance between a measure and its empirical version in Euclidean settings. 
Following this exposition, we transition directly to the proof of Theorem~\ref{thrm:ConcentrationFiniteMetricSpaces}. 

\begin{lemma}[Concentration of {the} Wasserstein {distance} in Euclidean setting] \label{lem:conc_Wassmetric} \hfill\\
Let $\xxx$ be a finite subset of 
$\mathbb{R}^m$, and let $(\xxx,\{0,1\}^{\xxx},\nu)$ be a Borel probability space. Let $X_1,\dots,X_N\sim\nu$ be i.i.d. random variables taking values in $\xxx$ and $\nu^N$ be the associated empirical measure.
Then for every $\epsilon>0$, the event
\begin{equation} \label{conc_Wassmetric1}
    \big|\mathcal{W}_{\xxx}(\nu,\nu^N)-\mathbb{E}\big[\mathcal{W}_{\xxx} (\nu,\nu^N)\big]\big|>\epsilon
\end{equation}
holds with probability at most $2e^{-2N\epsilon^2/\mathsf{d}(\xxx)^2}$. Moreover, 
\begin{equation} \label{conc_Wassmetric2}
    \mathbb{E}\big[\mathcal{W}_{\xxx}(\nu,\nu^N)\big]\leq \tilde{C}_m\mathsf{d}(\xxx)r_m(N),
\end{equation}
where $\tilde{C}_m,r_m(N)$ are given in Table~\ref{tab:concentration_main_RATES}.
\end{lemma}

\begin{proof}[Proof of Theorem~\ref{thrm:ConcentrationFiniteMetricSpaces}]
Let $m\in\mathbb{N}$. By Proposition~\ref{prop:EuclideanRep}, we may assume the existence of a bi-Lipschitz embedding $\varphi: \xxx \to \mathbb{R}^m$, where we omit the dimension $m$ in the notation of $\varphi$.
For simplicity, we denote the lower Lipschitz and the upper Lipschitz constants of $\varphi$ respectively as ${\rm L}_{\ell}$ and ${\rm L}_{\mathit{u}}$. Then for $x,\tilde{x}\in\xxx$
\begin{equation} \label{biLipschitz}
   {\rm L}_{\ell} d_{\xxx}(x,\tilde{x}) \leq \|\varphi(x)-\varphi(\tilde{x})\|_2
   \leq {\rm L}_{\mathit{u}} d_{\xxx}(x,\tilde{x}).
\end{equation}
Define the following push-forward probability measures on $\varphi(\xxx)$,
\begin{equation} \label{eqdef:pushforward}
    \mathbb{Q} \eqdef \varphi_{\texttt{\#}}\mathbb{P} \quad\text{ and }\quad
    \mathbb{Q}^N \eqdef \varphi_{\texttt{\#}}\mathbb{P}^N,
\end{equation}
where we have taken note in the second definition that the push-forward of an empirical measure yields another empirical measure. Since $\varphi(\xxx)\subset\mathbb{R}^m$, $(\varphi(\xxx),\|\cdot\|_2)$ itself constitutes a finite metric space. We will demonstrate that the Wasserstein distance between $\mathbb{Q},\mathbb{Q}^N$ in $\varphi(\xxx)$, denoted as $\mathcal{W}_{\varphi(\xxx)}(\mathbb{Q},\mathbb{Q}^N)$, is comparable to that between $\mathbb{P}, \mathbb{P}^N$ in $\xxx$, denoted as $\mathcal{W}_{\xxx}(\mathbb{P},\mathbb{P}^N)$. The importance of this comparability will become clear when we progress toward concluding the theorem.
 
Due to the injectivity of $\varphi$, the map $\varphi^{-1}:\varphi(\xxx)\to\xxx$ is well-defined. 
Thus, considering a coupling $\gamma\in\Gamma(\mathbb{P}, \mathbb{P}^N)$, we obtain from \eqref{eqdef:pushforward}
\begin{equation} \label{eq:obs1}
    \mathbb{Q}^N(\varphi(\{y\})) 
    = \mathbb{P}^N(\{y\}) = \sum_{x\in\xxx} \gamma(x,y)
    = \sum_{\varphi(x)\in\varphi(\xxx)} \gamma(\varphi^{-1}\circ\varphi(x),\varphi^{-1}\circ\varphi(y)),
\end{equation}
for any $y\in\xxx$, as well as
\begin{equation} \label{eq:obs2}
    \mathbb{Q}(\varphi(\{x\})) 
    = \sum_{\varphi(y)\in\varphi(\xxx)} \gamma(\varphi^{-1}\circ\varphi(x),\varphi^{-1}\circ\varphi(y)),
\end{equation}
for any $x\in\xxx$. 
These observations \eqref{eq:obs1}, \eqref{eq:obs2} enable us to identify the coupling $\gamma\in \Gamma(\mathbb{P}, \mathbb{P}^N)$ with the coupling $\gamma\circ\varphi^{-1}\in \Gamma(\varphi_{\texttt{\#}}\mathbb{P},\varphi_{\texttt{\#}}\mathbb{P}^N)=\Gamma(\mathbb{Q},\mathbb{Q}^N)$, which we define to be, based on the provided derivations,
\begin{equation} \label{identified}
    \gamma\circ\varphi^{-1}(\varphi(x),\varphi(y)) \eqdef \gamma(x,y).
\end{equation}
Conversely, we can identify any $\tilde{\gamma}\in\Gamma(\mathbb{Q},\mathbb{Q}^N)$ with $\tilde{\gamma}\circ\varphi\in \Gamma(\mathbb{P},\mathbb{P}^N)$, where
\begin{equation} \label{invidentified}
    \tilde{\gamma}\circ\varphi(x,y) \eqdef \tilde{\gamma}(\varphi(x),\varphi(y)).
\end{equation}
Next, by recalling the bi-Lipschitz property of $\varphi$ in \eqref{biLipschitz}, we infer that
\begin{align} \label{transform1}
    \sum_{x\in\xxx}\sum_{y\in\xxx} d_{\xxx}(x,y)\,\gamma(x,y) \leq \sum_{\varphi(x)\in \varphi(\xxx)} \sum_{\varphi(y)\in \varphi(\xxx)} {\rm L}_{\ell}^{-1} \|\varphi(x)-\varphi(y)\|_2\,\gamma(x,y),
\end{align}
where, considering \eqref{identified}, the right-hand-side term can also be written as
\begin{equation*}
    \sum_{\varphi(x)\in \varphi(\xxx)} \sum_{\varphi(y)\in \varphi(\xxx)} {\rm L}_{\ell}^{-1} \|\varphi(x)-\varphi(y)\|_2\,\gamma\circ\varphi^{-1}(\varphi(x),\varphi(y)).
\end{equation*}
As $\gamma$ runs over $\Gamma(\mathbb{P},\mathbb{P}^N)$, $\gamma\circ\varphi^{-1}$ exhausts $\Gamma(\mathbb{Q},\mathbb{Q}^N)$. Therefore, we deduce from \eqref{transform1} that
\begin{equation} \label{sideone}
    \mathcal{W}_{\xxx}(\mathbb{P},\mathbb{P}^N)\leq {\rm L}_{\ell}^{-1} \mathcal{W}_{\varphi(\xxx)}(\mathbb{Q},\mathbb{Q}^N).
\end{equation}
Deploying a similar line of reasoning and leveraging \eqref{biLipschitz}, \eqref{invidentified}, we arrive at
\begin{equation} \label{sidetwo}
    \mathcal{W}_{\varphi(\xxx)}(\mathbb{Q},\mathbb{Q}^N)\leq{\rm L}_{\mathit{u}}\mathcal{W}_{\xxx}(\mathbb{P},\mathbb{P}^N).
\end{equation}
Collectively, \eqref{sideone} and \eqref{sidetwo} summarize the comparability of the Wasserstein distances we wish to establish. 

We now address part (i) of the theorem. By utilizing Lemma~\ref{lem:conc_Wassmetric}, particularly \eqref{conc_Wassmetric2} with $\mathbb{Q}$ replacing $\nu$ and $\mathbb{Q}^N$ replacing $\nu^N$, alongside with \eqref{biLipschitz}, we obtain
\begin{align} \label{conclemmacite1}
    \mathbb{E}\big[\mathcal{W}_{\varphi(\xxx)}(\mathbb{Q},\mathbb{Q}^N)\big] 
    \leq \tilde{C}_m\mathsf{d}(\varphi(\xxx))r_m(N) 
    \leq \tilde{C}_m{\rm L}_{\mathit{u}}\mathsf{d}(\xxx)r_m(N).
\end{align}
Integrating the findings in \eqref{sideone}, \eqref{conclemmacite1}, we acquire (i), as desired.

To prove (ii), we apply \eqref{conc_Wassmetric1} of Lemma~\ref{lem:conc_Wassmetric}, which implies that, for each $\varepsilon>0$, the event
\begin{equation} \label{conc_highprob}
    -\varepsilon<\mathcal{W}_{\varphi(\xxx)}(\mathbb{Q},\mathbb{Q}^N)-\mathbb{E}\big[\mathcal{W}_{\varphi(\xxx)} (\mathbb{Q},\mathbb{Q}^N)\big] < \varepsilon
\end{equation}
holds with probability at least
\begin{equation} \label{problow}
    1-2\exp\Big(-\frac{2N\epsilon^2}{\mathsf{d}(\varphi(\xxx))^2}\Big) \geq 1-2\exp\Big(-\frac{2N\epsilon^2}{{\rm L}_{\mathit{u}}^2\mathsf{d}(\xxx)^2}\Big).
\end{equation}
Combining \eqref{conc_highprob} with \eqref{sideone}, \eqref{sidetwo}, \eqref{conclemmacite1} we derive
\begin{align} \label{oneside}
    \nonumber \mathcal{W}_{\xxx} (\mathbb{P},\mathbb{P}^N) - \mathbb{E}\big[\mathcal{W}_{\xxx}(\mathbb{P},\mathbb{P}^N)\big] &\leq {\rm L}_{\ell}^{-1}\mathcal{W}_{\varphi(\xxx)}(\mathbb{Q},\mathbb{Q}^N) - {\rm L}_{\mathit{u}}^{-1}\mathbb{E} \big[\mathcal{W}_{\varphi(\xxx)}(\mathbb{Q},\mathbb{Q}^N)\big]\\
    \nonumber &\leq {\rm L}_{\ell}^{-1}\mathbb{E} \big[\mathcal{W}_{\varphi(\xxx)}(\mathbb{Q},\mathbb{Q}^N)\big] + {\rm L}_{\ell}^{-1}\varepsilon -{\rm L}_{\mathit{u}}^{-1}\mathbb{E} \big[\mathcal{W}_{\varphi(\xxx)}(\mathbb{Q},\mathbb{Q}^N)\big]\\
    &\leq \tilde{C}_m(\tau(\varphi)-1)\mathsf{d}(\xxx)r_m(N) + {\rm L}_{\ell}^{-1}\varepsilon.    
\end{align}
Likewise, a similar line of reasoning leads to
\begin{equation} \label{theotherside}
    \mathbb{E} \big[\mathcal{W}_{\xxx}(\mathbb{P},\mathbb{P}^N)\big] -\mathcal{W}_{\xxx}(\mathbb{P},\mathbb{P}^N) \\
    \leq 
    \tilde{C}_m(\tau(\varphi)-1)\mathsf{d}(\xxx)r_m(N) + {\rm L}_{\mathit{u}}^{-1}\varepsilon.
\end{equation}
Taken together, \eqref{problow}, \eqref{oneside}, \eqref{theotherside} indicate that
\begin{equation*}
    \big|\mathcal{W}_{\xxx}(\mathbb{P},\mathbb{P}^N)-\mathbb{E}\big[\mathcal{W}_{\xxx}(\mathbb{P},\mathbb{P}^N)\big]\big|\\
    \le \tilde{C}_m(\tau(\varphi)-1)\mathsf{d}(\xxx)r_m(N) + {\rm L}_{\ell}^{-1}\varepsilon,
\end{equation*}
with probability at least $1-2\exp\big(-2N\epsilon^2/({\rm L}_{\mathit{u}}^2\mathsf{d}(\xxx)^2)\big)$. Letting $\tilde{\varepsilon}={\rm L}_{\ell}^{-1}\varepsilon$, we obtain conclusion (ii) of the theorem. 
\end{proof}

\subsection{Proof of Theorem~\ref{thm:MAIN_no_partition}} \label{sec:proof_MAIN_no_partition} 
In what follows, we simplify the notation of the diameter $\mathsf{d}(\xxx\times\yyy)$ to $\mathsf{d}$.

We define for each $\hat{f}\in\mathcal{F}_{L}$, $\mathcal{L}_{\hat{f}} \eqdef \mathcal{L}\circ(\hat{f}\times\mathrm{Id}_{\yyy})$, where $(\hat{f}\times\mathrm{Id}_{\yyy})(X,Y)=(\hat{f}(X),Y)$. 
Then, from \eqref{eqdef:mu}, an equivalent expression for $\mathcal{R}^{\star}(\hat{f})$ is
\begin{equation*} 
    \mathcal{R}^{\star}(\hat{f})
    = \mathbb{E}_{(X,Y)\sim\mu} \big[\mathcal{L}(\hat{f}(X),Y)\big]
    = \mathbb{E}_{(X,Y)\sim\mu} \big[\mathcal{L}_{\hat{f}}(X,Y)\big].
\end{equation*}
Similarly, by noting \eqref{eqdef:empmeas} and \eqref{eq:datagenerating}, we can reformulate $\hat{\mathcal{R}}(\hat{f})$ and $\mathcal{R}(\hat{f})$ respectively as
\begin{alignat*}{2}
    \hat{\mathcal{R}}(\hat{f}) 
    &= \mathbb{E}_{(X,Y)\sim\mathbb{P}^N} \big[\mathcal{L}(\hat{f}(X),Y)\big]
    &&= \mathbb{E}_{(X,Y)\sim\mathbb{P}^N} \big[\mathcal{L}_{\hat{f}}(X,Y)\big],\\
    \mathcal{R}(\hat{f}) 
    &= \mathbb{E}_{(X,Y)\sim\mathbb{P}} \big[\mathcal{L}(\hat{f}(X),Y)\big]
    &&= \mathbb{E}_{(X,Y)\sim\mathbb{P}} \big[\mathcal{L}_{\hat{f}}(X,Y)\big].
\end{alignat*}
Hence, for a given $\hat{f}\in\mathcal{F}_{L}$, we can represent the induced generalization gap to be
\begin{align} \label{eq:gengaprecall}
    \big|\mathcal{R}(\hat{f}) - \hat{\mathcal{R}}(\hat{f})\big| = \big|\mathbb{E}_{(X,Y)\sim\mathbb{P}}\big[\mathcal{L}_{\hat{f}}(X,Y)\big]
    - 
    \mathbb{E}_{(X',Y')\sim\mathbb{P}^N} \big[\mathcal{L}_{\hat{f}} (X',Y')\big]\big|,
\end{align}
and the induced {reconstruction} gap to be
\begin{align} \label{eq:estgaprecall}
    \big|\mathcal{R}^{\star}(\hat{f}) - \hat{\mathcal{R}}(\hat{f})\big| = \big|\mathbb{E}_{(X,Y)\sim\mu}\big[\mathcal{L}_{\hat{f}}(X,Y)\big]
    - 
    \mathbb{E}_{(X',Y')\sim\mathbb{P}^N} \big[\mathcal{L}_{\hat{f}} (X',Y')\big]\big|.
\end{align}
Observe that $\mathcal{L}_{\hat{f}}$ is Lipschitz with the upper Lipschitz constant bounded above by
\begin{equation*}
    {\rm L}_{\mathit{u}}(\mathcal{L}_{\hat{f}})
    \leq{\rm L}_{\mathit{u}}(\mathcal{L})\max\{1,{\rm L}_{\mathit{u}}(\hat{f})\}
    \leq{\rm L}_{\mathit{u}}(\mathcal{L})\max\{1,L\}=\bar{L}.
\end{equation*}
As such, an application of \citep[Lemma 12]{hou2022instance} to the induced generalization gap in \eqref{eq:gengaprecall} allows us to acquire
\begin{equation} \label{eq:true_diff}
    \begin{split}
        \big|\mathcal{R}(\hat{f}) - \hat{\mathcal{R}}(\hat{f})\big| 
        &= \big|\mathbb{E}_{(X,Y)\sim\mathbb{P}} \big[\mathcal{L}_{\hat{f}}(X,Y)\big] - \mathbb{E}_{(X',Y')\sim\mathbb{P}^N} \big[\mathcal{L}_{\hat{f}} (X',Y')\big]\big|\\
        &\le 
        {\rm L}_{\mathit{u}}(\mathcal{L}_{\hat{f}})\,\mathcal{W}_{\xxx\times\yyy}(\mathbb{P},\mathbb{P}^N)\\
        &\le 
        \bar{L}\, \mathcal{W}_{\xxx\times\yyy}(\mathbb{P},\mathbb{P}^N).
    \end{split}
\end{equation}
Taking the supremum of \eqref{eq:true_diff} over all $\hat{f}\in\mathcal{F}_{L}$ delivers
\begin{equation} \label{eq:truerisk_bound}
    \sup_{\hat{f}\in\mathcal{F}_{L}} \big|\mathcal{R}(\hat{f}) - \hat{\mathcal{R}}(\hat{f})\big|
    \le 
    \bar{L}\,\mathcal{W}_{\xxx\times\yyy}(\mathbb{P},\mathbb{P}^N).
\end{equation}
Likewise, through a similar application of \citep[Lemma B.1]{hou2022instance} to the induced {reconstruction} gap in \eqref{eq:estgaprecall}, we get
\begin{equation*}
    \big|\mathcal{R}^{\star}(\hat{f}) - \hat{\mathcal{R}}(\hat{f})\big| 
    \leq 
    \bar{L} \big(\mathcal{W}_{\xxx\times\yyy}(\mu,\mathbb{P}) + \mathcal{W}_{\xxx\times\yyy}(\mathbb{P},\mathbb{P}^N)\big),
\end{equation*}
which subsequently suggests, from \eqref{eq:Wass_TV} and assumption \eqref{eq:SNR},
\begin{equation} \label{eq:excessrisk_bound}
    \sup_{\hat{f}\in\mathcal{F}_{L}} \big|\mathcal{R}^{\star}(\hat{f}) - \hat{\mathcal{R}}(\hat{f})\big|
    \le 
    \bar{L} \big(\mathsf{d}\,\Delta + \mathcal{W}_{\xxx\times\yyy} (\mathbb{P},\mathbb{P}^N)\big).
\end{equation}
Having obtained \eqref{eq:truerisk_bound} and \eqref{eq:excessrisk_bound}, we now turn to constrain $\mathcal{W}_{\xxx\times\yyy}(\mathbb{P},\mathbb{P}^N)$ probabilistically.

Fix $m\in\mathbb{N}$. By Proposition~\ref{prop:EuclideanRep}, we may assume the existence of a bi-Lipschitz embedding, {denoted simply as} $\varphi$, that maps $\xxx\times\yyy$ into $\mathbb{R}^m$.
Thus, conclusion (i) of Theorem~\ref{thrm:ConcentrationFiniteMetricSpaces} implies
\begin{equation}
\label{eq:Bound_I}
    \mathbb{E}\big[\mathcal{W}_{\xxx\times\yyy}(\mathbb{P},\mathbb{P}^N)\big]
    \le 
    \tilde{C}_m\tau(\varphi)\mathsf{d}\, r_m(N).
\end{equation}
Moreover, conclusion (ii) of Theorem~\ref{thrm:ConcentrationFiniteMetricSpaces} asserts that for any $\varepsilon > 0$, the event
\begin{equation} \label{eq:Bound_II}
    \big|\mathcal{W}_{\xxx\times\yyy}(\mathbb{P},\mathbb{P}^N) - \mathbb{E}\big[\mathcal{W}_{\xxx\times\yyy}(\mathbb{P},\mathbb{P}^N)\big]\big| \leq 
    \tilde{C}_m (\tau(\varphi)-1)\mathsf{d}\, r_m(N) +\varepsilon
\end{equation}
holds with probability at least $1-2\exp\big(-2N\epsilon^2/(\tau(\varphi)^2\mathsf{d}^2)\big)$. Let $\delta\in (0,1]$ and set 
\begin{equation*}
    \varepsilon = \mathsf{d}\,\tau(\varphi)\sqrt{\log(2/\delta)}/\sqrt{2N} >0.
\end{equation*} 
Then by employing \eqref{eq:Bound_I}, \eqref{eq:Bound_II}, and the triangle inequality, we deduce the following result with probability at least $1-\delta$
\begin{align}\label{concentration_bound_W}  
    \mathcal{W}_{\xxx\times\yyy}(\mathbb{P},\mathbb{P}^N) \le 
    \tilde{C}_m (2\tau(\varphi)-1)\mathsf{d}\, r_m(N)
    + \mathsf{d}\,\tau(\varphi)\sqrt{\log(2/\delta)}/\sqrt{2N}.
\end{align} 
Finally, substituting \eqref{concentration_bound_W} back into \eqref{eq:truerisk_bound} and \eqref{eq:excessrisk_bound}, we assert that with probability at least $1-\delta$, the worst-case generalization gap $\sup_{\hat{f}\in\mathcal{F}_{L}} \big|\mathcal{R}(\hat{f}) - \hat{\mathcal{R}}(\hat{f})\big|$ is dominated by
\begin{equation*}
    \bar{L}\mathsf{d} \Big(\tilde{C}_m(2\tau(\varphi)-1)r_m(N) + \tau(\varphi) \sqrt{\log(2/\delta)}/\sqrt{2N}\Big),
\end{equation*}
and the worst-case {reconstruction} gap $\sup_{\hat{f}\in\mathcal{F}_{L}} \big|\mathcal{R}^{\star}(\hat{f}) - \hat{\mathcal{R}}(\hat{f})\big|$ by
\begin{equation*}
    \bar{L}\mathsf{d}\Big(\Delta + \tilde{C}_m(2\tau(\varphi)-1)r_m(N) + \tau(\varphi) \sqrt{\log(2/\delta)}/\sqrt{2N}\Big).
\end{equation*}
The proof is now complete. \qed


\appendix

\subsection{Proof of Corollary~\ref{cor:MLPDiscretetization}}
\label{s:Lipschitz_Bounds}
We begin with a lemma that quantifies the effect of any abstract rounding/discretization rule on the Lipschitz constants of the deep-learning hypothesis functions involved.
This will serve as the basis for deriving Corollary~\ref{cor:MLPDiscretetization}. 

Consider a finite metric space $(\xxx,d_{\xxx})$. Let $A\subset\xxx$. We define the minimal distance $\mathrm{sep}(A,d_{\xxx})$ between points in $A$ as follows. If $\mathrm{card}(A)\geq 2$, we let
\begin{equation*}
  \mathrm{sep}(A,d_{\xxx})
    \eqdef 
    \underset{\underset{x,z\in A }{x\neq z}}{\min} d_{\xxx}(x,z). 
\end{equation*}
If $\mathrm{card}(A)= 1$, then $\mathrm{sep}(A,d_{\xxx})=1$, and if $A=\emptyset$, then $\mathrm{sep}(A,d_{\xxx})=\infty$.

\begin{lemma}[Discretization effect on Lipschitz constants]
\label{lem:RoundingLemma_Abstract} \hfill\\
Let ${\eta_{\xxx}},{\eta_{\yyy}}>0$ and $L\ge 0$.
Let $(\xxx,d_{\xxx})$ and $(\yyy,d_{\yyy})$ be two metric spaces. Let $\tilde{\xxx}\subset \xxx$ be such that $\mathrm{sep}(\tilde{\xxx},d_{\xxx}) \geq  {\eta_{\xxx}}$, and let $\tilde{\yyy}$ be a ${\eta_{\yyy}}$-packing  
of $\yyy$. Let $\Pi: \yyy\to \tilde{\yyy}$ be the rounding map satisfying 
\begin{equation*}
    d_{\yyy}(\Pi(y),y) = \min_{u\in\tilde{\yyy}} d_{\yyy}(u,y).
\end{equation*}
Then for any $L$-Lipschitz map $f:\xxx\to \yyy$, the corresponding discretized map $\bar{f}:\tilde{\xxx}\to \tilde{\yyy}$, defined by $\bar{f}\eqdef \Pi\circ f|_{\tilde{\xxx}}$, is $(L+\frac{{\eta_{\yyy}}}{{\eta_{\xxx}}})$-Lipschitz.

In particular, if $f:\mathbb{R}^d\rightarrow \mathbb{R}$ is an $L$-Lipschitz map and $\bar{f}: \mathbb{R}^d_{p,M} \rightarrow \mathbb{R}^1_{p,M}$ is its discretized version, then $\bar{f}$ is $L+1$-Lipschitz.
\end{lemma}

\begin{proof} Let $x\not= z\in \tilde{\xxx}$. Then
\begin{align} \label{leg1}
    d_{\yyy}(\bar{f}(x),\bar{f}(z)) = d_{\yyy}(\Pi\circ f|_{\tilde{\xxx}}(x),\Pi\circ f|_{\tilde{\xxx}}(z)) = d_{\yyy}(\Pi\circ f(x),\Pi\circ f(z)).
\end{align}
By the triangle inequality,
\begin{equation} \label{leg21}
    \begin{split}
        d_{\yyy}(\Pi\circ f(x),\Pi\circ f(z)) &\leq d_{\yyy}(\Pi\circ f(x),f(x)) + d_{\yyy}(f(z),\Pi\circ f(z)) + d_{\yyy}(f(x),f(z))\\
        &\leq 2\max_{u\in \{f(z),f(x)\}}\,\min_{y\in\tilde{\yyy}} d_{\yyy}(y,u) + d_{\yyy}(f(x),f(z)).
    \end{split}
\end{equation}
On the one hand, drawing from \citep[page~98]{varderVaartWellner_EmpiricalProcessesBook_1996} that every ${\eta_{\yyy}}$-packing of a semi-metric space is a ${\eta_{\yyy}}/2$-covering \cite[Definition~13]{acciaio2023designing}, we obtain 
\begin{align} \label{leg22}
    2\max_{u\in \{f(z),f(x)\}}\,\min_{y\in\tilde{\yyy}} d_{\yyy}(y,u) \leq 2({\eta_{\yyy}}/2) \leq \frac{{\eta_{\xxx}}}{d_{\xxx}(x,z)} \, d_{\xxx}(x,z)\leq \frac{{\eta_{\xxx}}}{{\eta_{\yyy}}} \, d_{\xxx}(x,z).
\end{align}
On the other hand, it follows from the Lipschitz property of $f$ that
\begin{equation} \label{leg23}
    d_{\yyy}(f(x),f(z)) \leq Ld_{\xxx}(x,z).
\end{equation}
Hence, by reinserting \eqref{leg22}, \eqref{leg23} back in \eqref{leg21}, and combining the outcome with \eqref{leg1}, we arrive at
\begin{equation*}
    d_{\yyy}(\bar{f}(x),\bar{f}(z)) \leq \Big(L + \frac{{\eta_{\yyy}}}{{\eta_{\xxx}}}\Big)d_{\xxx}(x,z),
\end{equation*}
which is the first conclusion. 

The second conclusion is an application of the first, to the case where $\xxx=\mathbb{R}^d$, $\tilde{\xxx}=\mathbb{R}^d_{p,M}$, $\yyy=\mathbb{R}$, $\tilde{\yyy}=\mathbb{R}^1_{p,M}$, and $\Pi: \mathbb{R}\to\mathbb{R}^1_{p,M}$. It becomes evident upon noting that $\mathrm{sep}(\mathbb{R}^m_{p,M},\|\cdot\|_2)=2^{-p}$ for all $m\in\mathbb{N}$.
\end{proof}

\begin{proof}[Proof of Corollary~\ref{cor:MLPDiscretetization}]
Consider the map 
\begin{equation*}
    x\mapsto A^{(l)}\mathrm{ReLU}\bullet x + b^{(l)}
\end{equation*}
at the $l$th layer. Since $|A_{i,j}^{(l)}|\le \mathbf{B}$, the upper Lipschitz constant of this map is bounded above by $\mathbf{B}\mathbf{W}^2$. Consequently, for $f\in\mathfrak{L}_{\mathbf{W},\mathbf{L},\mathbf{B}}$
\begin{equation*}
   {\rm L}_{\mathit{u}}(f)\leq \mathbf{B}\mathbf{L}\mathbf{W}^2.
\end{equation*}
Upon substituting in the values of $\mathbf{W}, \mathbf{L}, \mathbf{B}$ specified in \eqref{eq:LB}, \eqref{eqdef:CFclass}, we deduce that if $f\in \mathfrak{L}_{6,\mathbf{L}+1,\mathbf{B}}$, then 
\begin{equation*}
   {\rm L}_{\mathit{u}}(f)\leq 72K(\lceil \Lambda^{\star}\rceil + 2)\mathbf{B}.
\end{equation*}
Therefore, as per definition \eqref{eqdef:CFclass} and Lemma~\ref{lem:RoundingLemma_Abstract}, it holds that
\begin{equation*}
    {\rm L}_{\mathit{u}}(\bar{f})\leq 72K(\lceil \Lambda^{\star}\rceil + 2)\mathbf{B} + 1,
\end{equation*}
whenever $\bar{f} = \Pi\circ f|_{\mathbb{R}^d_{p,M}}$ and $f\in\mathfrak{L}_{6,\mathbf{L}+1,\mathbf{B}}$.  
{By Assumption~\ref{ass:sparse_support}, we have 
\begin{equation}
\label{eq:cardinality}
        \mathrm{card}(\mathscr{X}\times \mathscr{Y}) = \mathrm{card}({\rm supp}(\mu_{\xxx})\times \mathbb{R}^1_{p,M})
        \le 
        Cd^s 2^{p+1}M.
\end{equation}
}
{Recognizing \eqref{eq:cardinality} and the fact that $\mathsf{d}(\xxx\times\yyy)\leq \mathsf{d}(\mathbb{R}^d_{p,M}\times \mathbb{R}^1_{p,M})\le 2(\sqrt{d}+1)M$, the corollary now follows from conclusion (i) of Theorem~\ref{thm:MAIN_no_partition}.}
\end{proof}

\subsection{The Noise Level $\Delta$ for an independent additive noise}
\label{a:FurtherDiscussion__ss:GeneralizedAdditiveNoise}

In classical non-parametric regression analysis, typically the image of the target function $f^{\star}$ is obscured by a centered additive noise. 
To demonstrate this, we provide a scenario. For simplicity, we consider two independent real-valued random variables {$f^{\star}(X)$, $Z$} with respective laws $\nu_{{f^{\star}(X)}},\nu_{{Z}}\in\mathcal{P}(\mathbb{R})$, both supported in $[0,1]$. Let {$f^{\star}(X)+Z$} represent a ``corruption" of {$f^{\star}(X)$} by an additive noise. It follows that {$f^{\star}(X)+Z$} has law $\nu_{{f^{\star}(X)}}\ast\nu_{{Z}}\in\mathcal{P}(\mathbb{R})$, whose support is confined to $[0,2]$. 
The $1$-Wasserstein distance between the ``true" law $\nu_{{f^{\star}(X)}}$ and its rendition $\nu_{{f^{\star}(X)}}\ast\nu_{{Z}}$, 
denoted by $\mathcal{W}_{\mathbb{R}}(\nu_{{f^{\star}(X)}}\ast\nu_{{Z}},\nu_{{f^{\star}(X)}})$, is
\begin{equation*} 
    \mathcal{W}_{\mathbb{R}}(\nu_{{f^{\star}(X)}}\ast\nu_{{Z}},\nu_{{f^{\star}(X)}}) \eqdef \inf_{\gamma\in\Gamma(\nu_{{f^{\star}(X)}}\ast\nu_{{Z}},\nu_{{f^{\star}(X)}})} \iint_{\mathbb{R}\times\mathbb{R}} |x-y|\,\mathrm{d}\gamma(x,y).
\end{equation*}
The total variation $\mathrm{TV}(\nu_{{f^{\star}(X)}}\ast\nu_{{Z}},\nu_{{f^{\star}(X)}})$ is {(note the similarity with \eqref{eq:TVfinite})}
\begin{equation*} 
   \mathrm{TV}(\nu_{{f^{\star}(X)}}\ast\nu_{{Z}},\nu_{{f^{\star}(X)}}) \eqdef \sup_{A\subset\mathbb{R}} \big|\nu_{{f^{\star}(X)}}\ast\nu_{{Z}}(A)-\nu_{{f^{\star}(X)}}(A)\big|.
\end{equation*}
Since the interval $[0,2]$ has length $2$, and since \eqref{eq:Wass_TV} still holds in this case, we obtain
\begin{equation*} 
    \mathcal{W}_{\mathbb{R}}(\nu_{{f^{\star}(X)}}\ast\nu_{{Z}},\nu_{{f^{\star}(X)}})
    \leq 
    2\mathrm{TV}(\nu_{{f^{\star}(X)}}\ast\nu_{{Z}},\nu_{{f^{\star}(X)}}) {\leq 2\Delta};
\end{equation*}
i.e., the Wasserstein distance between $\nu_{{f^{\star}(X)}}$ and its corrupted version is constrained by their total variation, scaled by the diameter of the support.

\section*{Acknowledgments}
A.K.\ is supported by NSERC Discovery grant No.\ RGPIN-2023-04482; A.M.N.\ is supported by the Austrian Science Fund (FWF): P 37010, and G.P.\ is supported by the ETH Z\"{u}rich foundation.

The authors would like to thank Ben Bolker for his helpful feedback and references on floating point arithmetic, and Hassan Ashtiani, Alireza Fathollah Pour, and Songyan Hou for their insightful discussions and helpful feedback. A.K. would also like to thank Noah Forman for the inspirational discussion on the Library of Babel, which sparked some initial ideas.

\bibliographystyle{IEEEtran}
\bibliography{Bookeeping/3_References}

\end{document}

%% file: Bookeeping/0_Information.tex
\usepackage{lipsum}
\usepackage{amsfonts}
\usepackage{graphicx}
\usepackage{epstopdf}
\usepackage{algorithmic}
\ifpdf
  \DeclareGraphicsExtensions{.eps,.pdf,.png,.jpg}
\else
  \DeclareGraphicsExtensions{.eps}
\fi

\usepackage{amsopn}


%% file: Bookeeping/1_macros_and_packages.tex
\usepackage{float}\usepackage{colortbl}
\usepackage{amssymb, mathrsfs}
\usepackage{mathtools}
\usepackage{placeins}
\usepackage{xparse}
\usepackage{color}
\usepackage{comment}

\usepackage[numbers]{natbib}

\usepackage{hyperref}
    \hypersetup{
    	colorlinks = true,
    	linkcolor = blue,
    	anchorcolor = blue,
    	citecolor = blue,
    	filecolor = blue,
    	urlcolor = blue
    }

\usepackage{accents}

\usepackage{adjustbox}
\usepackage{booktabs}
\usepackage{tablefootnote}
\usepackage{lscape}

\usepackage{marginnote}
\usepackage{xcolor}
\definecolor{darkcyan}{rgb}{0.0, 0.55, 0.55}
\definecolor{MidnightBlue}{RGB}{25,25,112}
\definecolor{MidnightBlueComplementingGreen}{RGB}{25,112,25}
\definecolor{MidnightBlueComplementingPurple}{RGB}{112,25,112}
\definecolor{MidnightBlueComplementingRed}{RGB}{112,25,69}
\definecolor{WowColor}{rgb}{.75,0,.75}
\definecolor{MildlyAlarming}{rgb}{0.85,0.25,0.1}
\definecolor{SubtleColor}{rgb}{0,0,.50}
\definecolor{antiquefuchsia}{rgb}{0.57, 0.36, 0.51}
\definecolor{fashionfuchsia}{rgb}{0.96, 0.0, 0.63}
\definecolor{jade}{rgb}{0.0, 0.66, 0.42}
\definecolor{caribbeangreen}{rgb}{0.0, 0.8, 0.6}
\definecolor{aquamarine}{rgb}{0.5, 0.8, 0.85}
\definecolor{attentioncolor}{RGB}{152,90,81}
\definecolor{burgred}{RGB}{40,3,22}
\definecolor{AnnieGreen}{RGB}{17,123,92}
\definecolor{Turquoise}{RGB}{64,224,208}

\definecolor{darkjade}{RGB}{0,122,84}
\definecolor{Window1}{RGB}{92,150,31}%
    \definecolor{Window1dark}{RGB}{41,67,13}%
\definecolor{Window2}{RGB}{255,168,28}
    \definecolor{Window2dark}{RGB}{114,75,12}
\definecolor{Window3}{RGB}{255,96,33}
    \definecolor{Window3dark}{RGB}{97,36,12}
\definecolor{InputColor}{RGB}{20,255,177}
    \definecolor{InputColorlight}{RGB}{222,237,229}

\usepackage[colorinlistoftodos]{todonotes}

\NewDocumentCommand{\Noah}{mo}{
    \IfValueF{#2}{
                        {{
                            \textcolor{magenta}{ 
                            \textbf{N:}
                            \textit{{#1}}
                            }
                        }}
        }
    \IfValueT{#2}{
                        \marginnote{{\scriptsize
                            \textcolor{magenta}{ 
                            \textbf{N:}
                            \textit{{#1}}
                            }
                        }}
        }
                    }
\NewDocumentCommand{\Annie}{mo}{
    \IfValueF{#2}{
                        {{
                            \textcolor{AnnieGreen}{ 
                            \textbf{A:}
                            \textit{{#1}}
                            }
                        }}
        }
    \IfValueT{#2}{
                        \marginnote{{\scriptsize
                            \textcolor{AnnieGreen}{ 
                            \textbf{A:}
                            \textit{{#1}}
                            }
                        }}
        }
                    }
\NewDocumentCommand{\Gudi}{mo}{
    \IfValueF{#2}{
                        {{
                            \textcolor{caribbeangreen}{ 
                            \textbf{G:}
                            \textit{{#1}}
                            }
                        }}
        }
    \IfValueT{#2}{
                        \marginnote{{\scriptsize
                            \textcolor{caribbeangreen}{ 
                            \textbf{G:}
                            \textit{{#1}}
                            }
                        }}
        }
                    }

\NewDocumentCommand{\Songyan}{mo}{
    \IfValueF{#2}{
                        {{
                            \textcolor{violet}{ 
                            \textbf{Songyan:}
                            \textit{{#1}}
                            }
                        }}
        }
    \IfValueT{#2}{
                        \marginnote{{\scriptsize
                            \textcolor{violet}{ 
                            \textbf{Songyan:}
                            \textit{{#1}}
                            }
                        }}
        }
                    }

\NewDocumentCommand{\Alireza}{mo}{
    \IfValueF{#2}{
                        {{
                            \textcolor{blue}{ 
                            \textbf{Alireza:}
                            \textit{{#1}}
                            }
                        }}
        }
    \IfValueT{#2}{
                        \marginnote{{\scriptsize
                            \textcolor{blue}{ 
                            \textbf{N:}
                            \textit{{#1}}
                            }
                        }}
        }
                    }
\NewDocumentCommand{\Martina}{mo}{
    \IfValueF{#2}{
                        {{
                            \textcolor{red}{ 
                            \textbf{M:}
                            \textit{{#1}}
                            }
                        }}
        }
    \IfValueT{#2}{
                        \marginnote{{\scriptsize
                            \textcolor{red}{ 
                            \textbf{M:}
                            \textit{{#1}}
                            }
                        }}
        }
                    }

\usepackage{mathtools}
\usepackage{hyperref}
\makeatletter
\newcommand{\mytag}[2]{%
  \text{#1}%
  \@bsphack
  \begingroup
    \@onelevel@sanitize\@currentlabelname
    \edef\@currentlabelname{%
      \expandafter\strip@period\@currentlabelname\relax.\relax\@@@%
    }%
    \protected@write\@auxout{}{%
      \string\newlabel{#2}{%
        {#1}%
        {\thepage}%
        {\@currentlabelname}%
        {\@currentHref}{}%
      }%
    }%
  \endgroup
  \@esphack
}
\makeatother

\usepackage{cleveref}
\newcounter{termcounter}
\renewcommand{\thetermcounter}{\Roman{termcounter}}
\crefname{term}{term}{terms}
\creflabelformat{term}{#2\textup{(#1)}#3}

\makeatletter
\def\term{\@ifnextchar[\term@optarg\term@noarg}
\def\term@optarg[#1]#2{%
  \textup{#1}%
  \def\@currentlabel{#1}%
  \def\cref@currentlabel{[][2147483647][]#1}%
  \cref@label[term]{#2}}
\def\term@noarg#1{%
  \refstepcounter{termcounter}%
  \textup{(\thetermcounter)}%
  \cref@label[term]{#1}}
\makeatother

%% file: Bookeeping/2_mathcommands.tex
\newcommand{\eqdef}{\ensuremath{
        \overset{
                \scalebox{.5}{\mbox{def.}}
            }{=}
}}

\newcommand{\xxx}{\mathscr{X}}
\newcommand{\yyy}{\mathscr{Y}}

\DeclareMathOperator{\supp}{supp}
\newcommand{\bes}{\begin{subequations}}
\newcommand{\ees}{\end{subequations}}
\newcommand{\eea}{\end{eqnarray}}

\usepackage{bbm}

\renewcommand{\epsilon}{\varepsilon}

\newtheorem{assumptions}{Assumptions}
\newtheorem{InformalTheorem}{Informal Theorem}

\newtheorem{lemma}{Lemma}
\newtheorem{theorem}{Theorem}
\newtheorem{corollary}{Corollary}
\newtheorem{proposition}{Proposition}
\newtheorem{remark}{Remark}


%% file: Bookeeping/3_References.bib
@article{bartlett2002rademacher,
  title={Rademacher and gaussian complexities: Risk bounds and structural results},
  author={Bartlett, Peter L and Mendelson, Shahar},
  journal={Journal of machine learning research},
  volume={3},
  number={Nov},
  pages={463--482},
  year={2002}
}

@article {GottliebAryehKrauthgamer2016TCS,
    AUTHOR = {Gottlieb, Lee-Ad and Kontorovich, Aryeh and Krauthgamer,
              Robert},
     TITLE = {Adaptive metric dimensionality reduction},
   JOURNAL = {Theoret. Comput. Sci.},
  FJOURNAL = {Theoretical Computer Science},
    VOLUME = {620},
      YEAR = {2016},
     PAGES = {105--118},
      ISSN = {0304-3975,1879-2294},
   MRCLASS = {68Q32 (62H25 62H30)},
  MRNUMBER = {3461825},
       DOI = {10.1016/j.tcs.2015.10.040},
       URL = {https://doi.org/10.1016/j.tcs.2015.10.040},
}

@article{pfluger2010spatially,
  title={Spatially adaptive sparse grids for high-dimensional data-driven problems},
  author={Pfl{\"u}ger, Dirk and Peherstorfer, Benjamin and Bungartz, Hans-Joachim},
  journal={Journal of Complexity},
  volume={26},
  number={5},
  pages={508--522},
  year={2010},
  publisher={Elsevier}
}

@article{gross2024sparse,
  title={Sparse spectral methods for solving high-dimensional and multiscale elliptic PDEs},
  author={Gross, Craig and Iwen, Mark},
  journal={Foundations of Computational Mathematics},
  pages={1--47},
  year={2024},
  publisher={Springer}
}

@article{choi2021sparse,
  title={Sparse harmonic transforms: a new class of sublinear-time algorithms for learning functions of many variables},
  author={Choi, Bosu and Iwen, Mark A and Krahmer, Felix},
  journal={Foundations of Computational Mathematics},
  volume={21},
  number={2},
  pages={275--329},
  year={2021},
  publisher={Springer}
}

@article{sason2016f,
  title={$ f $-divergence Inequalities},
  author={Sason, Igal and Verd{\'u}, Sergio},
  journal={IEEE Transactions on Information Theory},
  volume={62},
  number={11},
  pages={5973--6006},
  year={2016},
  publisher={IEEE}
}

@article {IEEE_GeneralizationMLP_2023,
    AUTHOR = {Imaizumi, Masaaki and Schmidt-Hieber, Johannes},
     TITLE = {On generalization bounds for deep networks based on loss surface implicit regularization},
   JOURNAL = {IEEE Trans. Inform. Theory},
  FJOURNAL = {Institute of Electrical and Electronics Engineers.
              Transactions on Information Theory},
    VOLUME = {69},
      YEAR = {2023},
    NUMBER = {2},
     PAGES = {1203--1223},
      ISSN = {0018-9448,1557-9654},
   MRCLASS = {62M45 (68T07)},
  MRNUMBER = {4564651},
}

@inproceedings{aminian2021information,
  title={Information-theoretic bounds on the moments of the generalization error of learning algorithms},
  author={Aminian, Gholamali and Toni, Laura and Rodrigues, Miguel RD},
  booktitle={2021 IEEE International Symposium on Information Theory (ISIT)},
  pages={682--687},
  year={2021},
  organization={IEEE}
}

@inproceedings{zhou2023exactly,
  title={Exactly tight information-theoretic generalization error bound for the quadratic gaussian problem},
  author={Zhou, Ruida and Tian, Chao and Liu, Tie},
  booktitle={2023 IEEE International Symposium on Information Theory (ISIT)},
  pages={903--908},
  year={2023},
  organization={IEEE}
}

@article {IEEE_Generalization_2022,
    AUTHOR = {Zhou, Ruida and Tian, Chao and Liu, Tie},
     TITLE = {Individually conditional individual mutual information bound
              on generalization error},
   JOURNAL = {IEEE Trans. Inform. Theory},
  FJOURNAL = {Institute of Electrical and Electronics Engineers.
              Transactions on Information Theory},
    VOLUME = {68},
      YEAR = {2022},
    NUMBER = {5},
     PAGES = {3304--3316},
      ISSN = {0018-9448,1557-9654},
   MRCLASS = {62B10},
  MRNUMBER = {4433222},
       DOI = {10.1109/tit.2022.3144615},
       URL = {https://doi.org/10.1109/tit.2022.3144615},
}

@article {IEEE_Generalization_2021,
    AUTHOR = {Esposito, Amedeo Roberto and Gastpar, Michael and Issa,
              Ibrahim},
     TITLE = {Generalization error bounds via {R}\'{e}nyi-,
              {$f$}-divergences and maximal leakage},
   JOURNAL = {IEEE Trans. Inform. Theory},
  FJOURNAL = {Institute of Electrical and Electronics Engineers.
              Transactions on Information Theory},
    VOLUME = {67},
      YEAR = {2021},
    NUMBER = {8},
     PAGES = {4986--5004},
      ISSN = {0018-9448,1557-9654},
   MRCLASS = {94A17 (60E15)},
  MRNUMBER = {4306311},
MRREVIEWER = {Chanchal\ Kundu},
       DOI = {10.1109/TIT.2021.3085190},
       URL = {https://doi.org/10.1109/TIT.2021.3085190},
}

@book{muller2018handbook,
  title={Handbook of floating-point arithmetic},
  author={Muller, Jean-Michel and Brisebarre, Nicolas and De Dinechin, Florent and Jeannerod, Claude-Pierre and Lefevre, Vincent and Melquiond, Guillaume and Revol, Nathalie and Stehl{\'e}, Damien and Torres, Serge and others},
  year={2018},
  publisher={Springer}
}

@article{GoldbergFloatingPointArithmitic_ACM_1991,
author = {Goldberg, David},
title = {What Every Computer Scientist Should Know about Floating-Point Arithmetic},
year = {1991},
issue_date = {March 1991},
publisher = {Association for Computing Machinery},
address = {New York, NY, USA},
volume = {23},
number = {1},
issn = {0360-0300},
url = {https://doi.org/10.1145/103162.103163},
doi = {10.1145/103162.103163},
journal = {ACM Comput. Surv.},
month = {mar},
pages = {5–48},
numpages = {44}
}

@article{park2024expressive,
  title={Expressive Power of ReLU and Step Networks under Floating-Point Operations},
  author={Park, Yeachan and Hwang, Geonho and Lee, Wonyeol and Park, Sejun},
  journal={arXiv preprint arXiv:2401.15121},
  year={2024}
}

@article{xu2020cot,
  title={Cot-gan: Generating sequential data via causal optimal transport},
  author={Xu, Tianlin and Wenliang, Li Kevin and Munn, Michael and Acciaio, Beatrice},
  journal={Advances in neural information processing systems},
  volume={33},
  pages={8798--8809},
  year={2020}
}

@book{lorentz1996constructive,
    AUTHOR = {Lorentz, George G. and Golitschek, Manfred v. and Makovoz,
              Yuly},
     TITLE = {Constructive approximation},
    SERIES = {Grundlehren der mathematischen Wissenschaften [Fundamental
              Principles of Mathematical Sciences]},
    VOLUME = {304},
      NOTE = {Advanced problems},
 PUBLISHER = {Springer-Verlag, Berlin},
      YEAR = {1996},
     PAGES = {xii+649},
      ISBN = {3-540-57028-4},
   MRCLASS = {41-02 (41-XX)},
  MRNUMBER = {1393437},
MRREVIEWER = {L.\ L.\ Schumaker},
       DOI = {10.1007/978-3-642-60932-9},
       URL = {https://doi.org/10.1007/978-3-642-60932-9},
}

@article{cao2019multi,
  title={Multi-marginal wasserstein gan},
  author={Cao, Jiezhang and Mo, Langyuan and Zhang, Yifan and Jia, Kui and Shen, Chunhua and Tan, Mingkui},
  journal={Advances in Neural Information Processing Systems},
  volume={32},
  year={2019}
}

@inproceedings{arjovsky2017wasserstein,
  title={Wasserstein generative adversarial networks},
  author={Arjovsky, Martin and Chintala, Soumith and Bottou, L{\'e}on},
  booktitle={International conference on machine learning},
  pages={214--223},
  year={2017},
  organization={PMLR}
}

@article{zhang2018link,
  title={Link prediction based on graph neural networks},
  author={Zhang, Muhan and Chen, Yixin},
  journal={Advances in neural information processing systems},
  volume={31},
  year={2018}
}

@inproceedings{herrera2020neural,
  title={Neural Jump Ordinary Differential Equations: Consistent Continuous-Time Prediction and Filtering},
  author={Herrera, Calypso and Krach, Florian and Teichmann, Josef},
  booktitle={International Conference on Learning Representations},
  year={2020}
}

@inproceedings{morrill2021neural,
  title={Neural rough differential equations for long time series},
  author={Morrill, James and Salvi, Cristopher and Kidger, Patrick and Foster, James},
  booktitle={International Conference on Machine Learning},
  pages={7829--7838},
  year={2021},
  organization={PMLR}
}

@article{chen2018neural,
  title={Neural ordinary differential equations},
  author={Chen, Ricky TQ and Rubanova, Yulia and Bettencourt, Jesse and Duvenaud, David K},
  journal={Advances in neural information processing systems},
  volume={31},
  year={2018}
}

@article{kovachki2021universal,
  title={On universal approximation and error bounds for Fourier neural operators},
  author={Kovachki, Nikola and Lanthaler, Samuel and Mishra, Siddhartha},
  journal={The Journal of Machine Learning Research},
  volume={22},
  number={1},
  pages={13237--13312},
  year={2021},
  publisher={JMLRORG}
}

@article {MR4583284,
    AUTHOR = {Tsigler, Alexander and Bartlett, Peter L.},
     TITLE = {Benign overfitting in ridge regression},
   JOURNAL = {J. Mach. Learn. Res.},
  FJOURNAL = {Journal of Machine Learning Research (JMLR)},
    VOLUME = {24},
      YEAR = {2023},
     PAGES = {Paper No. [123], 76},
      ISSN = {1532-4435,1533-7928},
   MRCLASS = {62J07 (68T07)},
  MRNUMBER = {4583284},
}

@article{barzilai2023generalization,
  title={Generalization in Kernel Regression Under Realistic Assumptions},
  author={Barzilai, Daniel and Shamir, Ohad},
  journal={arXiv preprint arXiv:2312.15995},
  year={2023}
}

@article{cheng2024,
      title={Characterizing Overfitting in Kernel Ridgeless Regression Through the Eigenspectrum}, 
      author={Tin Sum Cheng and Aurelien Lucchi and Anastasis Kratsios and David Belius},
      year={2024},
      eprint={2402.01297},
      archivePrefix={arXiv},
      primaryClass={cs.LG},
    journal={arxiv},
}

@article{karniadakis2021physics,
  title={Physics-informed machine learning},
  author={Karniadakis, George Em and Kevrekidis, Ioannis G and Lu, Lu and Perdikaris, Paris and Wang, Sifan and Yang, Liu},
  journal={Nature Reviews Physics},
  volume={3},
  number={6},
  pages={422--440},
  year={2021},
  publisher={Nature Publishing Group UK London}
}

@inproceedings{attias2023optimal,
  title={Optimal Learners for Realizable Regression: PAC Learning and Online Learning},
  author={Attias, Idan and Hanneke, Steve and Kalavasis, Alkis and Karbasi, Amin and Velegkas, Grigoris},
  booktitle={Thirty-seventh Conference on Neural Information Processing Systems},
  year={2023}
}

@article {GraphKlotzVoigtlaender_2023_RDTDL__FOCM,
    AUTHOR = {Grohs, Philipp and Klotz, Andreas and Voigtlaender, Felix},
     TITLE = {Phase transitions in rate distortion theory and deep learning},
   JOURNAL = {Found. Comput. Math.},
  FJOURNAL = {Foundations of Computational Mathematics. The Journal of the
              Society for the Foundations of Computational Mathematics},
    VOLUME = {23},
      YEAR = {2023},
    NUMBER = {1},
     PAGES = {329--392},
      ISSN = {1615-3375,1615-3383},
   MRCLASS = {94A34 (28C20 41A46 68T07)},
  MRNUMBER = {4546149},
       DOI = {10.1007/s10208-021-09546-4},
       URL = {https://doi.org/10.1007/s10208-021-09546-4},
}

@article {DaubechiesDeVore_ContinuousToDigitalSamplingQuantization_2003__AnnMath,
    AUTHOR = {Daubechies, Ingrid and DeVore, Ron},
     TITLE = {Approximating a bandlimited function using very coarsely
              quantized data: a family of stable sigma-delta modulators of
              arbitrary order},
   JOURNAL = {Ann. of Math. (2)},
  FJOURNAL = {Annals of Mathematics. Second Series},
    VOLUME = {158},
      YEAR = {2003},
    NUMBER = {2},
     PAGES = {679--710},
      ISSN = {0003-486X,1939-8980},
   MRCLASS = {42C40 (41A30 94A12)},
  MRNUMBER = {2018933},
MRREVIEWER = {Gitta\ Kutyniok},
       DOI = {10.4007/annals.2003.158.679},
       URL = {https://doi.org/10.4007/annals.2003.158.679},
}

@article {DGunturk_ImprovedRatesContinuousToDigitalSamplingQuantization_2003__AMS,
    AUTHOR = {G\"{u}nt\"{u}rk, C. Sinan},
     TITLE = {Approximating a bandlimited function using very coarsely
              quantized data: improved error estimates in sigma-delta
              modulation},
   JOURNAL = {J. Amer. Math. Soc.},
  FJOURNAL = {Journal of the American Mathematical Society},
    VOLUME = {17},
      YEAR = {2004},
    NUMBER = {1},
     PAGES = {229--242},
      ISSN = {0894-0347,1088-6834},
   MRCLASS = {41A30 (94A20)},
  MRNUMBER = {2015335},
MRREVIEWER = {Alexander\ P.\ Petukhov},
       DOI = {10.1090/S0894-0347-03-00436-3},
       URL = {https://doi.org/10.1090/S0894-0347-03-00436-3},
}

@article {BartlettMendelson_2002,
    AUTHOR = {Bartlett, Peter L. and Mendelson, Shahar},
     TITLE = {Rademacher and {G}aussian  complexities: risk bounds and structural results},
   JOURNAL = {J. Mach. Learn. Res.},
  FJOURNAL = {Journal of Machine Learning Research (JMLR)},
    VOLUME = {3},
      YEAR = {2002},
    NUMBER = {Spec. Issue Comput. Learn. Theory},
     PAGES = {463--482},
      ISSN = {1532-4435},
   MRCLASS = {68T05 (60E15)},
  MRNUMBER = {1984026},
MRREVIEWER = {M. Iosifescu},
       DOI = {10.1162/153244303321897690},
       URL = {https://doi.org/10.1162/153244303321897690},
}

@article{daubechies2022nonlinear,
  title={Nonlinear approximation and (deep) ReLU networks},
  author={Daubechies, Ingrid and DeVore, Ronald and Foucart, Simon and Hanin, Boris and Petrova, Guergana},
  journal={Constructive Approximation},
  volume={55},
  number={1},
  pages={127--172},
  year={2022},
  publisher={Springer}
}

@article{acciaio2023designing,
  title={Designing universal causal deep learning models: The geometric (Hyper) transformer},
  author={Acciaio, Beatrice and Kratsios, Anastasis and Pammer, Gudmund},
  journal={Mathematical Finance},
  year={2023},
  publisher={Wiley Online Library}
}

@article{JMLRBartlett_HarveyLiawMehrabian_RiskBounds,
  author  = {Peter L. Bartlett and Nick Harvey and Christopher Liaw and Abbas Mehrabian},
  title   = {Nearly-tight VC-dimension and Pseudodimension Bounds for Piecewise Linear Neural Networks},
  journal = {Journal of Machine Learning Research},
  year    = {2019},
  volume  = {20},
  number  = {63},
  pages   = {1--17},
  url     = {http://jmlr.org/papers/v20/17-612.html}
}

@article{bartlett2017spectrally,
  title={Spectrally-normalized margin bounds for neural networks},
  author={Bartlett, Peter L and Foster, Dylan J and Telgarsky, Matus J},
  journal={Advances in neural information processing systems},
  volume={30},
  year={2017}
}

@article{neyshabur2017pac,
  title={A pac-bayesian approach to spectrally-normalized margin bounds for neural networks},
  author={Neyshabur, Behnam and Bhojanapalli, Srinadh and Srebro, Nathan},
  journal={arXiv preprint arXiv:1707.09564},
  year={2017}
}

@article{boche2022limitations,
  title={Limitations of deep learning for inverse problems on digital hardware},
  author={Boche, Holger and Fono, Adalbert and Kutyniok, Gitta},
  journal={arXiv preprint arXiv:2202.13490},
  year={2022}
}

@article{boche2022inverse,
  title={Inverse problems are solvable on real number signal processing hardware},
  author={Boche, Holger and Fono, Adalbert and Kutyniok, Gitta},
  journal={arXiv preprint arXiv:2204.02066},
  year={2022}
}

@article{boche2022non,
  title={Non-Computability of the Pseudoinverse on Digital Computers},
  author={Boche, Holger and Fono, Adalbert and Kutyniok, Gitta},
  journal={arXiv preprint arXiv:2212.02940},
  year={2022}
}

@article{yarotsky2017error,
  title={Error bounds for approximations with deep ReLU networks},
  author={Yarotsky, Dmitry},
  journal={Neural Networks},
  volume={94},
  pages={103--114},
  year={2017},
  publisher={Elsevier}
}

@article{mhaskar2016deep,
  title={Deep vs. shallow networks: An approximation theory perspective},
  author={Mhaskar, Hrushikesh N and Poggio, Tomaso},
  journal={Analysis and Applications},
  volume={14},
  number={06},
  pages={829--848},
  year={2016},
  publisher={World Scientific}
}

@article{shen2022optimal,
  title={Optimal approximation rate of ReLU networks in terms of width and depth},
  author={Shen, Zuowei and Yang, Haizhao and Zhang, Shijun},
  journal={Journal de Math{\'e}matiques Pures et Appliqu{\'e}es},
  volume={157},
  pages={101--135},
  year={2022},
  publisher={Elsevier}
}

@article{cheng2023theoretical,
  title={A Theoretical Analysis of the Test Error of Finite-Rank Kernel Ridge Regression},
  author={Cheng, Tin Sum and Lucchi, Aurelien and Dokmani{\'c}, Ivan and Kratsios, Anastasis and Belius, David},
  journal = {Advances in Neural Information Processing Systems},
  year={2023}
}

@book {varderVaartWellner_EmpiricalProcessesBook_1996,
    AUTHOR = {van der Vaart, Aad W. and Wellner, Jon A.},
     TITLE = {Weak convergence and empirical processes},
    SERIES = {Springer Series in Statistics},
      NOTE = {With applications to statistics},
 PUBLISHER = {Springer-Verlag, New York},
      YEAR = {1996},
     PAGES = {xvi+508},
      ISBN = {0-387-94640-3},
   MRCLASS = {60F05 (60B12 62G30)},
  MRNUMBER = {1385671},
MRREVIEWER = {Miguel\ A.\ Arcones},
       DOI = {10.1007/978-1-4757-2545-2},
       URL = {https://doi.org/10.1007/978-1-4757-2545-2},
}

@article{dziugaitecomputing,
  title={Computing Nonvacuous Generalization Bounds for Deep (Stochastic) Neural Networks with Many More Parameters than Training Data},
  author={Dziugaite, Gintare Karolina and Roy, Daniel M},
 journal={Uncertainty in Artificial Intelligence},
year={2017}
}

@book{RomanhighDimensionalProbBook,
    AUTHOR = {Vershynin, Roman},
     TITLE = {High-dimensional probability},
    SERIES = {Cambridge Series in Statistical and Probabilistic Mathematics},
    VOLUME = {47},
      NOTE = {An introduction with applications in data science,
              With a foreword by Sara van de Geer},
 PUBLISHER = {Cambridge University Press, Cambridge},
      YEAR = {2018},
     PAGES = {xiv+284},
      ISBN = {978-1-108-41519-4},
   MRCLASS = {60-01 (60B05 60B20 60E15 60Fxx 62H25)},
  MRNUMBER = {3837109}
}

@book{shalev2014understanding,
  title={Understanding machine learning: From theory to algorithms},
  author={Shalev-Shwartz, Shai and Ben-David, Shai},
  year={2014},
  publisher={Cambridge university press}
}

@inproceedings{neyshabur2018pac,
  title={A {PAC}-{B}ayesian approach to spectrally-normalized margin bounds for neural networks},
  author={Neyshabur, Behnam and Bhojanapalli, Srinadh and Srebro, Nathan},
  booktitle={iclr},
  year={2018}
}

@article {NgYip_AA_2023__GNNsEigenRisk,
    AUTHOR = {Ng, Michael K. and Yip, Andy},
     TITLE = {Stability and generalization of graph convolutional networks
              in eigen-domains},
   JOURNAL = {Anal. Appl. (Singap.)},
  FJOURNAL = {Analysis and Applications},
    VOLUME = {21},
      YEAR = {2023},
    NUMBER = {3},
     PAGES = {819--840},
      ISSN = {0219-5305},
   MRCLASS = {68T05 (68Q32 90C31)},
  MRNUMBER = {4570455},
       DOI = {10.1142/S0219530523500021},
       URL = {https://doi.org/10.1142/S0219530523500021},
}

@inproceedings {Long_1997ComplexityPAC_ML1998,
    AUTHOR = {Long, Philip M.},
     TITLE = {The complexity of learning according to two models of a
              drifting environment},
 BOOKTITLE = {Proceedings of the {E}leventh {A}nnual {C}onference on
              {C}omputational {L}earning {T}heory ({M}adison, {WI}, 1998)},
     PAGES = {116--125},
 PUBLISHER = {ACM, New York},
      YEAR = {1998},
   MRCLASS = {68Q32 (68T05)},
  MRNUMBER = {1811576},
       DOI = {10.1145/279943.279968},
       URL = {https://doi.org/10.1145/279943.279968},
}

@article {AryehIosif_2019_ExactPACMLFiniteCase_2019,
    AUTHOR = {Kontorovich, Aryeh and Pinelis, Iosif},
     TITLE = {Exact lower bounds for the agnostic
              probably-approximately-correct ({PAC}) machine learning model},
   JOURNAL = {Ann. Statist.},
  FJOURNAL = {The Annals of Statistics},
    VOLUME = {47},
      YEAR = {2019},
    NUMBER = {5},
     PAGES = {2822--2854},
      ISSN = {0090-5364},
   MRCLASS = {62H30 (60E15 62C10 62C20 62G10 62G20 68T05)},
  MRNUMBER = {3988774},
       DOI = {10.1214/18-AOS1766},
       URL = {https://doi.org/10.1214/18-AOS1766},
}

@article {Talagrand_SharperBoundEmpGausProcess_1994,
    AUTHOR = {Talagrand, M.},
     TITLE = {Sharper bounds for {G}aussian and empirical processes},
   JOURNAL = {Ann. Probab.},
  FJOURNAL = {The Annals of Probability},
    VOLUME = {22},
      YEAR = {1994},
    NUMBER = {1},
     PAGES = {28--76},
      ISSN = {0091-1798},
   MRCLASS = {60G50 (60E99 60G17 62E99)},
  MRNUMBER = {1258865},
MRREVIEWER = {Evarist Gin\'{e}},
       URL =
              {http://links.jstor.org/sici?sici=0091-1798(199401)22:1<28:SBFGAE>2.0.CO;2-W&origin=MSN},
}

@article {AjtaiKomlosTusnady_1984__Combinatoria,
    AUTHOR = {Ajtai, M. and Koml\'{o}s, J. and Tusn\'{a}dy, G.},
     TITLE = {On optimal matchings},
   JOURNAL = {Combinatorica},
  FJOURNAL = {Combinatorica. An International Journal of the J\'{a}nos
              Bolyai Mathematical Society},
    VOLUME = {4},
      YEAR = {1984},
    NUMBER = {4},
     PAGES = {259--264},
      ISSN = {0209-9683},
   MRCLASS = {60D05 (60G40 90B15)},
  MRNUMBER = {779885},
MRREVIEWER = {Richard\ A.\ Vitale},
       DOI = {10.1007/BF02579135},
       URL = {https://doi.org/10.1007/BF02579135},
}

@article {WeedBach_Concentration_2019__BernoulliOptimal,
    AUTHOR = {Weed, Jonathan and Bach, Francis},
     TITLE = {Sharp asymptotic and finite-sample rates of convergence of
              empirical measures in {W}asserstein distance},
   JOURNAL = {Bernoulli},
  FJOURNAL = {Bernoulli. Official Journal of the Bernoulli Society for
              Mathematical Statistics and Probability},
    VOLUME = {25},
      YEAR = {2019},
    NUMBER = {4A},
     PAGES = {2620--2648},
      ISSN = {1350-7265,1573-9759},
   MRCLASS = {60B10 (62G30)},
  MRNUMBER = {4003560},
MRREVIEWER = {Aihua\ Xia},
       DOI = {10.3150/18-BEJ1065},
       URL = {https://doi.org/10.3150/18-BEJ1065},
}

@article{gonon2023approximation,
  title={Approximation bounds for random neural networks and reservoir systems},
  author={Gonon, Lukas and Grigoryeva, Lyudmila and Ortega, Juan-Pablo},
  journal={The Annals of Applied Probability},
  volume={33},
  number={1},
  pages={28--69},
  year={2023},
  publisher={Institute of Mathematical Statistics}
}

@article{galimberti2022designing,
  title={Designing universal causal deep learning models: The case of infinite-dimensional dynamical systems from stochastic analysis},
  author={Galimberti, Luca and Kratsios, Anastasis and Livieri, Giulia},
  journal={arXiv preprint arXiv:2210.13300},
  year={2022}
}

@article {MeiMontanariCPA_2022,
    AUTHOR = {Mei, Song and Montanari, Andrea},
     TITLE = {The generalization error of random features regression:
              precise asymptotics and the double descent curve},
   JOURNAL = {Comm. Pure Appl. Math.},
  FJOURNAL = {Communications on Pure and Applied Mathematics},
    VOLUME = {75},
      YEAR = {2022},
    NUMBER = {4},
     PAGES = {667--766},
      ISSN = {0010-3640},
   MRCLASS = {62-01 (62G08 62M45)},
  MRNUMBER = {4400901},
       DOI = {10.1002/cpa.22008},
       URL = {https://doi.org/10.1002/cpa.22008},
}

@article {BartlettBousquetMendelson_LocRadCompl_AnnStat_2005,
    AUTHOR = {Bartlett, Peter L. and Bousquet, Olivier and Mendelson,
              Shahar},
     TITLE = {Local {R}ademacher complexities},
   JOURNAL = {Ann. Statist.},
  FJOURNAL = {The Annals of Statistics},
    VOLUME = {33},
      YEAR = {2005},
    NUMBER = {4},
     PAGES = {1497--1537},
      ISSN = {0090-5364,2168-8966},
   MRCLASS = {62G08 (62H30 68Q32)},
  MRNUMBER = {2166554},
MRREVIEWER = {Arnak\ S.\ Dalalyan},
       DOI = {10.1214/009053605000000282},
       URL = {https://doi.org/10.1214/009053605000000282},
}

@article {Matouvek_OptimalEuclidean_Npoint,
    AUTHOR = {Matou\v{s}ek, Ji\v{r}\'{\i}},
     TITLE = {Bi-{L}ipschitz embeddings into low-dimensional {E}uclidean
              spaces},
   JOURNAL = {Comment. Math. Univ. Carolin.},
  FJOURNAL = {Commentationes Mathematicae Universitatis Carolinae},
    VOLUME = {31},
      YEAR = {1990},
    NUMBER = {3},
     PAGES = {589--600},
      ISSN = {0010-2628},
   MRCLASS = {54E35 (54C25)},
  MRNUMBER = {1078491},
}

@book {Matouvek_LecturesDiscreteGeometry_2002,
    AUTHOR = {Matou\v{s}ek, Ji\v{r}\'{\i}},
     TITLE = {Lectures on discrete geometry},
    SERIES = {Graduate Texts in Mathematics},
    VOLUME = {212},
 PUBLISHER = {Springer-Verlag, New York},
      YEAR = {2002},
     PAGES = {xvi+481},
      ISBN = {0-387-95373-6},
   MRCLASS = {52Cxx (52-01)},
  MRNUMBER = {1899299},
MRREVIEWER = {E. Hertel},
       DOI = {10.1007/978-1-4613-0039-7},
       URL = {https://doi.org/10.1007/978-1-4613-0039-7},
}

@article {Matouvsek_1996_Embeddings,
    AUTHOR = {Matou\v{s}ek, Ji\v{r}\'{\i}},
     TITLE = {On the distortion required for embedding finite metric spaces into normed spaces},
   JOURNAL = {Israel J. Math.},
  FJOURNAL = {Israel Journal of Mathematics},
    VOLUME = {93},
      YEAR = {1996},
     PAGES = {333--344},
      ISSN = {0021-2172},
   MRCLASS = {54E40 (05C10)},
  MRNUMBER = {1380650},
}

@article {BourgainEmbedding_Original_1985,
    AUTHOR = {Bourgain, J.},
     TITLE = {On {L}ipschitz embedding of finite metric spaces in {H}ilbert
              space},
   JOURNAL = {Israel J. Math.},
  FJOURNAL = {Israel Journal of Mathematics},
    VOLUME = {52},
      YEAR = {1985},
    NUMBER = {1-2},
     PAGES = {46--52},
      ISSN = {0021-2172},
   MRCLASS = {46B99 (05C80 58B05)},
  MRNUMBER = {815600},
MRREVIEWER = {Michel Deza},
       DOI = {10.1007/BF02776078},
       URL = {https://doi.org/10.1007/BF02776078},
}

@book {DubhashiPanconesi_2009_Concentration,
	AUTHOR = {Dubhashi, Devdatt P. and Panconesi, Alessandro},
	TITLE = {Concentration of measure for the analysis of randomized
	algorithms},
	PUBLISHER = {Cambridge University Press, Cambridge},
	YEAR = {2009},
	PAGES = {xvi+196},
	ISBN = {978-0-521-88427-3},
	MRCLASS = {68W20 (60C05 60F10 68-02 68W40)},
	MRNUMBER = {2547432},
	MRREVIEWER = {Yannis C. Stamatiou},
	DOI = {10.1017/CBO9780511581274},
	URL = {https://doi-org.libaccess.lib.mcmaster.ca/10.1017/CBO9780511581274},
}

@incollection {JohnsonLindenstrauss_1984_originalpaper,
    AUTHOR = {Johnson, William B. and Lindenstrauss, Joram},
     TITLE = {Extensions of {L}ipschitz mappings into a {H}ilbert space},
 BOOKTITLE = {Conference in modern analysis and probability ({N}ew {H}aven,
              {C}onn., 1982)},
    SERIES = {Contemp. Math.},
    VOLUME = {26},
     PAGES = {189--206},
 PUBLISHER = {Amer. Math. Soc., Providence, RI},
      YEAR = {1984},
   MRCLASS = {46B20},
  MRNUMBER = {737400},
MRREVIEWER = {Yehoram Gordon},
       DOI = {10.1090/conm/026/737400},
       URL = {https://doi.org/10.1090/conm/026/737400},
}

@article {MardiajiaoTanczosNowakWeissman__2020_IAI_ConcentrationEmpiricalKL,
    AUTHOR = {Mardia, Jay and Jiao, Jiantao and T\'{a}nczos, Ervin and Nowak,
              Robert D. and Weissman, Tsachy},
     TITLE = {Concentration inequalities for the empirical distribution of discrete distributions: beyond the method of types},
   JOURNAL = {Inf. Inference},
  FJOURNAL = {Information and Inference. A Journal of the IMA},
    VOLUME = {9},
      YEAR = {2020},
    NUMBER = {4},
     PAGES = {813--850},
      ISSN = {2049-8764},
   MRCLASS = {60E15 (60F10 94A17)},
  MRNUMBER = {4188228},
       DOI = {10.1093/imaiai/iaz025},
       URL = {https://doi.org/10.1093/imaiai/iaz025},
}

@article {Kloeckner_2020_CounterCurse,
    AUTHOR = {Kloeckner, Beno\^{i}t},
     TITLE = {Empirical measures: regularity is a counter-curse to
              dimensionality},
   JOURNAL = {ESAIM Probab. Stat.},
  FJOURNAL = {ESAIM. Probability and Statistics},
    VOLUME = {24},
      YEAR = {2020},
     PAGES = {408--434},
      ISSN = {1292-8100},
   MRCLASS = {60B10 (49Q20 60J05 62E17)},
  MRNUMBER = {4153634},
       DOI = {10.1051/ps/2019025},
       URL = {https://doi.org/10.1051/ps/2019025},
}

@article{sommerfeld2018inference,
  title={Inference for empirical Wasserstein distances on finite spaces},
  author={Sommerfeld, Max and Munk, Axel},
  journal={Journal of the Royal Statistical Society. Series B (Statistical Methodology)},
  volume={80},
  number={1},
  pages={219--238},
  year={2018},
  publisher={JSTOR}
}

@article{kratsios2022small,
  title={Small Transformers Compute Universal Metric Embeddings},
  author={Kratsios, Anastasis and Debarnot, Valentin and Dokmani{\'c}, Ivan},
  journal={Journal of Machine Learning Research},
  year={2023}
}

@article{fournier2015rate,
  title={On the rate of convergence in Wasserstein distance of the empirical measure},
  author={Fournier, Nicolas and Guillin, Arnaud},
  journal={Probability theory and related fields},
  volume={162},
  number={3-4},
  pages={707--738},
  year={2015},
  publisher={Springer}
}

@article{boissard2014mean,
  title={On the mean speed of convergence of empirical and occupation measures in Wasserstein distance},
  author={Boissard, Emmanuel and Le Gouic, Thibaut},
  journal={Annales de l'IHP Probabilit{\'e}s et statistiques},
  volume={50},
  number={2},
  pages={539--563},
  year={2014}
}

@article{kantorovich1958space,
  title={On a space of totally additive functions, vestn},
  author={Kantorovich, L and Rubinstein, G},
  journal={Vestnik Leningrad. Univ},
  year={1958}
}

@article {Kloeckner_2012_QuantizationAlhforsRegularity,
    AUTHOR = {Kloeckner, Beno\^{\i}t},
     TITLE = {Approximation by finitely supported measures},
   JOURNAL = {ESAIM Control Optim. Calc. Var.},
  FJOURNAL = {ESAIM. Control, Optimisation and Calculus of Variations},
    VOLUME = {18},
      YEAR = {2012},
    NUMBER = {2},
     PAGES = {343--359},
      ISSN = {1292-8119},
   MRCLASS = {49Q20 (41A60 60B05)},
  MRNUMBER = {2954629},
MRREVIEWER = {Berardo Ruffini},
       DOI = {10.1051/cocv/2010100},
       URL = {https://doi.org/10.1051/cocv/2010100},
}

@article {LiuPages_Quantization_JMLR,
    AUTHOR = {Liu, Yating and Pag\`es, Gilles},
     TITLE = {Convergence rate of optimal quantization and application to
              the clustering performance of the empirical measure},
   JOURNAL = {J. Mach. Learn. Res.},
  FJOURNAL = {Journal of Machine Learning Research (JMLR)},
    VOLUME = {21},
      YEAR = {2020},
     PAGES = {Paper No. 86, 36},
      ISSN = {1532-4435,1533-7928},
   MRCLASS = {62E17 (62H30)},
  MRNUMBER = {4119154},
MRREVIEWER = {Shun\ Matsuura},
}

@book {GrafLuschgy_2000_FoundationsQuantizationofProbMeasure,
    AUTHOR = {Graf, Siegfried and Luschgy, Harald},
     TITLE = {Foundations of quantization for probability distributions},
    SERIES = {Lecture Notes in Mathematics},
    VOLUME = {1730},
 PUBLISHER = {Springer-Verlag, Berlin},
      YEAR = {2000},
     PAGES = {x+230},
      ISBN = {3-540-67394-6},
   MRCLASS = {60E99 (60F25 62H05 94A12)},
  MRNUMBER = {1764176},
MRREVIEWER = {Kalev P\"{a}rna},
       DOI = {10.1007/BFb0103945},
       URL = {https://doi.org/10.1007/BFb0103945},
}

@article{hou2022instance,
  title={Instance-Dependent Generalization Bounds via Optimal Transport},
  author={Hou, Songyan and Kassraie, Parnian and Kratsios, Anastasis and Krause, Andreas and Rothfuss, Jonas},
   JOURNAL = {J. Mach. Learn. Res.},
  FJOURNAL = {Journal of Machine Learning Research (JMLR)},
  year={2023}
}

@book {VillaniOTBook_2009,
    AUTHOR = {Villani, C\'{e}dric},
     TITLE = {Optimal transport},
    SERIES = {Grundlehren der mathematischen Wissenschaften [Fundamental
              Principles of Mathematical Sciences]},
    VOLUME = {338},
      NOTE = {Old and new},
 PUBLISHER = {Springer-Verlag, Berlin},
      YEAR = {2009},
     PAGES = {xxii+973},
      ISBN = {978-3-540-71049-3},
   MRCLASS = {49-02 (28A75 37J50 49Q20 53C23 58E30)},
  MRNUMBER = {2459454},
MRREVIEWER = {Dario Cordero-Erausquin},
       DOI = {10.1007/978-3-540-71050-9},
       URL = {https://doi.org/10.1007/978-3-540-71050-9},
}

@book{Matouvsek_2002_LecturesDiscreteGeo,
    AUTHOR = {Matou\v{s}ek, Ji\v{r}\'{\i}},
     TITLE = {Lectures on discrete geometry},
    SERIES = {Graduate Texts in Mathematics},
    VOLUME = {212},
 PUBLISHER = {Springer-Verlag, New York},
      YEAR = {2002},
     PAGES = {xvi+481},
}

@article{durand2021least,
  title={The least doubling constant of a path graph},
  author={Durand-Cartagena, Estibalitz and Soria, Javier and Tradacete, Pedro},
  journal={arXiv preprint arXiv:2111.09196},
  year={2021}
}

@article {DurandCartagenaEstibalitzSoraTradacete_2023_DiscreteMath,
    AUTHOR = {Durand-Cartagena, Estibalitz and Soria, Javier and Tradacete,
              Pedro},
     TITLE = {Doubling constants and spectral theory on graphs},
   JOURNAL = {Discrete Math.},
  FJOURNAL = {Discrete Mathematics},
    VOLUME = {346},
      YEAR = {2023},
    NUMBER = {6},
     PAGES = {Paper No. 113354, 17},
      ISSN = {0012-365X,1872-681X},
   MRCLASS = {05C62 (05C50)},
  MRNUMBER = {4546047},
MRREVIEWER = {Serge\u{\i}\ V.\ Konyagin},
       DOI = {10.1016/j.disc.2023.113354},
       URL = {https://doi.org/10.1016/j.disc.2023.113354},
}

@article{petersen2018optimal,
  title={Optimal approximation of piecewise smooth functions using deep ReLU neural networks},
  author={Petersen, Philipp and Voigtlaender, Felix},
  journal={Neural Networks},
  volume={108},
  pages={296--330},
  year={2018},
  publisher={Elsevier}
}
